\documentclass[conference]{IEEEtran}
\IEEEoverridecommandlockouts

\usepackage{amsthm}
\newtheorem{theorem}{Theorem}[section]

\newtheorem{prop}[theorem]{Proposition}
\newtheorem{assumption}[theorem]{Assumption}
\newtheorem{remark}{Remark}
\newtheorem{problem}{Problem}
\newtheorem{definition}[theorem]{Definition}

\newtheorem{exmp}{Example}[section]

\usepackage{cite}
\usepackage{algorithm}
\usepackage[noend]{algpseudocode}
\usepackage{multirow}
\usepackage{color}
\usepackage{amsfonts}
\usepackage{mysymbol}
\usepackage[dvipsnames]{xcolor}
\usepackage[hyphens]{url}
\usepackage{subcaption}
\usepackage{amsmath}
\usepackage{amssymb}

\DeclareMathOperator{\sinc}{sinc}

\allowdisplaybreaks
\newcommand{\vertiii}[1]{{\left\vert\kern-0.25ex\left\vert\kern-0.25ex\left\vert #1
    \right\vert\kern-0.25ex\right\vert\kern-0.25ex\right\vert}}
\usepackage[breaklinks=true, colorlinks, bookmarks=true, citecolor=Black, urlcolor=Violet,linkcolor=Black]{hyperref}

\usepackage{comment}
\usepackage[colorinlistoftodos,prependcaption,textwidth=1.5cm,textsize=tiny]{todonotes}

\usepackage[font={small}]{caption}

\def\BibTeX{{\rm B\kern-.05em{\sc i\kern-.025em b}\kern-.08em
    T\kern-.1667em\lower.7ex\hbox{E}\kern-.125emX}}

\begin{document}
% \DeclareMathOperator{\sinc}{sinc}

% \title{Conference Paper Title*\\

\title{Verified Compositional Neuro-Symbolic Control for Stochastic Systems with Temporal Logic Tasks
}

\author{%
Jun Wang$^{1}$, Haojun Chen$^{2}$, Zihe Sun$^{2}$, and Yiannis Kantaros$^{1}$% <-this % stops a space
\thanks{$^{1}$J. Wang and Y. Kantaros are with the Department of Electrical and Systems Engineering, Washington University in St. Louis (WashU), St. Louis, MO, USA.
{\tt\small \{junw,ioannisk\}@wustl.edu}}%}%
\thanks{$^{2}$H. Chen and Z. Sun are with the Department of Mechanical Engineering, Washington University in St. Louis (WashU), St. Louis, MO, USA.
{\tt\small \{c.haojun, s.zihe\}@wustl.edu}}%}%
\thanks{This work was supported in part by the NSF award CNS $\#2231257$.}
}

\maketitle

\begin{abstract}
Several methods have been proposed recently to learn neural network (NN) controllers for autonomous agents, with unknown and stochastic dynamics, tasked with complex missions captured by Linear Temporal Logic (LTL). Due to the sample-inefficiency of the majority of these works, compositional learning methods have been proposed decomposing the LTL specification into smaller sub-tasks. Then, separate controllers are learned and composed to satisfy the original task. A key challenge within these approaches is that they often lack safety guarantees or the provided guarantees are impractical. This paper aims to address this challenge. Particularly, we consider autonomous systems with unknown and stochastic dynamics and LTL-encoded tasks. We assume that the system is equipped with a finite set of base skills modeled by trained NN feedback controllers. Our goal is to check if there exists a temporal composition of the trained NN controllers - and if so, to compute it -  that will yield a composite system behavior that satisfies the assigned LTL task with probability one. We propose a new approach that relies on a novel integration of automata theory and data-driven reachability analysis tools for NN-controlled stochastic systems. The resulting neuro-symbolic controller allows the agent to generate safe behaviors for unseen complex temporal logic tasks in a zero-shot fashion by leveraging its base skills. We show correctness of the proposed method and we provide conditions under which it is complete. To the best of our knowledge, this is the first work that designs \textit{verified} temporal compositions of NN controllers for unknown and stochastic systems. 
Finally, we provide extensive numerical simulations and hardware experiments on robot navigation tasks to demonstrate the proposed method.
\end{abstract}


\section{Introduction}
Machine learning (ML) techniques have been successfully applied to learn controllers for systems with highly nonlinear, stochastic or unknown dynamics and complex tasks including reach-avoid, sequencing, surveillance, and delivery.  Linear Temporal Logic (LTL) formulas has been used extensively recently to formally define temporal and logical requirements that exist in such complex tasks \cite{baier2008principles}. Learning controllers for LTL specifications has been studied in \cite{wang2015temporal,hahn,gao2019reduced,bouton2019reinforcement,hasanbeig2019reinforcement,bozkurt2020control,lavaei2020formal,hasanbeig2022lcrl,cai2021reinforcement,kantaros2022accelerated} and the references therein. 
The learned controllers are typically modelled using neural networks (NN) to account for continuous state and action spaces. Despite the impressive empirical performance of NN-driven systems, they are characterized by the following limitations: (i) they often lack safety and robustness guarantees, especially, when applied to unseen tasks; and (ii) they are sample-inefficient to learn which becomes more pronounced as the task complexity increases.
%as underscored by recent studies \cite{huang2017adversarial}.

%. 

%, modelled 

%sutton2018reinforcement,haarnoja2018learning,
%Typically, in RL, control objectives are specified as reward functions. %, which , however, may not elicit the desired system behavior. 
%However, specifying reward-based objectives can be highly non-intuitive, especially for complex tasks, while poorly designed rewards can significantly compromise system performance \cite{dewey2014reinforcement}. To address this challenge, Linear Temporal logic (LTL) has recently been employed to specify %RL tasks
%. LTL provides a natural encoding of 
%tasks  that would have been very hard to define using Markovian rewards \cite{baier2008principles}; e.g., consider a navigation task requiring to visit regions of interest in a specific order.

%Several model-free RL methods with LTL-encoded tasks have been proposed recently; see e.g., 
%Recently, several methods have been proposed to train neural network (NN) controllers for autonomous systems such as deep reinforcement learning (RL) \cite{gao2019reduced} and model predictive control (MPC) \cite{rubies2019classification}. 
%

% 
%

To address the first limitation, several verification and analysis methods have been proposed for NN-controlled systems. These works can be categorized into open-loop \cite{fazlyab2020safety, katz2017reluplex, Dutta2017OutputRA, ruan2018reachability, batten2021efficient, chen2021deepsplit,baluta2021scalable} and closed-loop \cite{huang2019reachnn, sun2019formal, Hu2020ReachSDPRA, Tran2020NNVTN, dutta2019reachability, ivanov2021verisig, sibai2021mathsf, ivanov2020verifying, tran2019safety}. Open-loop verification methods aim to verify properties of a \text{trained NN}. For instance, \cite{Dutta2017OutputRA} addresses the problem of output range analysis of trained NNs given an input set. Closed-loop ones verification algorithms investigate if a dynamical system with a trained NN controller can satisfy a  reach-avoid property given a set of possible initial states. For instance, \cite{ivanov2021verisig} considers hybrid systems driven by feedforward NN controllers with sigmoid or tanh activation functions. Similarly, \cite{Hu2020ReachSDPRA} considers discrete-time linear time-varying systems with NN feedback controllers with arbitrary non-linear activation functions. Shielding mechanisms have also been proposed to ensure safety during the training and/or the deployment phase \cite{carr2022safe,Bastani2021SafeRL}. 
The main idea in the latter set of works is to compute a backup policy (called shield) that overrides the nominal control policy when the system is about to violate safety constraints. Although safety is ensured in these works, liveness guarantees are not provided. 
%
%
%Formal methods have also been leveraged recently to design RL controllers \cite{alshiekh2018safe, gao2019variance, hasanbeig2019reinforcement, bouton2019reinforcement, bozkurt2020control, cai2021reinforcement,kantaros2022accelerated}. Particularly, the main goal in these works is to learn a controller that maximizes the probability of satisfying complex properties captured by formal languages such as Linear Temporal Logic (LTL) \cite{baier2008principles,fainekos2005hybrid,smith2011optimal,leahy2016persistent,shoukry2017linear,kantaros2018text, wang2021verifying}.
%
%Nevertheless, the designed controllers are not supported by any safety or robustness guarantees when function approximations, such as NNs, are used to model them.
%
%To address the lack of safety, robustness, and/or mission performance guarantees of NN controllers, 
 
%Common in the works discussed above is that they focus on learning \textit{a single} NN controller that satisfies a user-specified mission and safety property while safety/performance guarantees are provided by verification engines. 
To address the second limitation, compositional learning methods have been proposed recently \cite{tasse2022skill,  jothimurugan2021compositional, kuo2021compositional, neary2022verifiable}. 
%Common in the works discussed above is that they focus on learning \textit{a single} NN controller that satisfies a user-specified mission and safety property while safety/performance guarantees are provided by verification engines. However, the sample complexity and computational cost of learning a single controller increases drastically as the task complexity increases. Motivated by these limitations, compositional learning methods have been proposed recently \cite{tasse2022skill,  jothimurugan2021compositional, kuo2021compositional, neary2022verifiable}. 
%
The key idea in these works is to decompose the overarching LTL-encoded mission into simpler sub-tasks for which NN controllers can be learned (e.g., using RL) more efficiently. Then, these NN controllers are composed to satisfy the original task.
Nevertheless, the resulting controllers often either lack safety guarantees or the provided guarantees are impractical (e.g., lower bounds on satisfaction probability) for safety-critical applications; also, these guarantees are typically specific to a fixed and given initial system state.
Related is also the recent work in \cite{ivanov2021composelearning} where,  unlike \cite{jothimurugan2021compositional,kuo2021compositional,tasse2022skill,neary2022verifiable}, selection of sub-tasks remain outside the robot control and, instead, they are revealed by the environment; e.g., a `turn left' sub-task is revealed as a robot navigates a hallway. To the contrary, in \cite{jothimurugan2021compositional,kuo2021compositional,tasse2022skill,neary2022verifiable}
sub-tasks are `strategically' selected by the system to satisfy a temporal logic task.
We note that guarantees for a given composite controller can possibly be provided using the above verification methods. Nevertheless, the latter would require knowledge of the system dynamics is known. %; as a result, they cannot be applied straightforwardly to cases where only a black-box simulator of the real system is available. 
Additionally, even if the system dynamics is known, these verification methods cannot provide guidance in terms of how the base NN controllers should be composed to ensure verified task satisfaction.

To bridge this gap, in this paper, we propose a new method to design \textit{verified} temporal compositions of trained NN controllers for temporal logic tasks for agents with unknown and stochastic dynamics. 
%
% Specifically, we consider stochastic systems tasked with complex high-level missions captured by a fragment of Linear Temporal Logic (LTL), called co-safe LTL \cite{baier2008principles}. We  consider systems with known nonlinear dynamics that are subject to unknown but upper-bounded exogenous disturbances. 
Specifically, we consider autonomous systems/agents that are governed by unknown discrete-time nonlinear dynamics while being subject to unknown, but upper-bounded, exogenous disturbances. %We do not assume that the system dynamics is known. Instead, w
We assume that we have access to a simulator that can generate system trajectories; this is a standard assumption e.g., in RL algorithms. 
The agent is tasked with complex high-level missions captured by a fragment of LTL, called co-safe LTL \cite{baier2008principles}. 
The LTL formula requires the agent to reach and avoid certain system states in a temporal/logical order. For instance, an LTL formula may model an aerial delivery task requiring a drone to deliver packages in location A, B, and C, in this specific order, while avoiding collisions with phone poles and flying over residential areas. We assume that the agent is equipped with a finite set of trained NN controllers where each controller can lead the agent towards a specific state while avoiding certain sub-spaces. We emphasize that we do not make any assumptions about how these NNs have been trained or about their accuracy and architecture. For instance, they may have been trained using RL \cite{sutton2018reinforcement} or MPC-based methods \cite{rubies2019classification}.
Our goal is to check if there exists a temporal composition of the trained NN controllers - and if so, to compute it - that will yield a composite system behavior satisfying the LTL task with probability one, for all initial states belonging to a given set. 

To address this problem, we propose a new approach that relies on a novel integration of existing tools such as  automata, graph-search methods, and data-driven reachability analysis for NN-driven systems. Particularly, inspired by recent compositional temporal logic planning methods \cite{bisoffi2018hybrid, kantaros2020reactive, jothimurugan2021compositional}, we propose a symbolic reasoning method that leverages automaton representations of LTL formulas and graph-search methods to decompose LTL tasks into sequences of reach-avoid tasks. Then, we employ a recently proposed reachability analysis tool \cite{lew2022simple} to check if there exists at least one temporal composition of the NN controllers that can satisfy with probability one all sub-tasks captured by at least one sequence of tasks.
The resulting neuro-symbolic controller allows the agent to generate verified behaviours for unseen complex temporal logic tasks in a zero-shot fashion by leveraging its current NN-based skills.
We show that our proposed algorithm is correct, i.e., the controller generated by the proposed method accomplishes the assigned task with probability one.  We note that, in general, the proposed algorithm is not complete in the sense that it may not find a control policy satisfying the assigned task with probability one, even though such a policy exists. We provide conditions under which our approach is complete and we discuss trade-offs between completeness and computational efficiency.
We highlight that our method can handle any other open-loop or feedback controllers that are not necessarily modeled as NNs. %by appropriate selection of the reachability analysis method.
In this paper, we consider NN controllers due to their brittleness to small input perturbations 
%huang2017adversarial
\cite{donti2020enforcing,xiong2020robustness} and their lack of safe generalization to unseen tasks.  We demonstrate the efficiency of the proposed method using extensive numerical experiments that involve navigation tasks for ground vehicles. %Particularly, we %apply our proposed method for various LTL specifications %using recent reachability analysis tools for NN-driven systems, such as $\epsilon$-RandUP \cite{lew2022simple},
%to 
%demonstrate trade-offs between completeness of the proposed algorithm, conservatism of reachability analysis, and runtime performance.
%
Also, we perform hardware experiments with ground robots to investigate the benefits and limitations of the proposed method with respect to the sim2real gap, i.e., the difference between modeled and real performance. Specifically, we evaluate performance of the designed composite controllers against un-modeled sources of uncertainty, such as floor slipperiness.

A preliminary version of this work was presented in our previous work \cite{wang2022verified} which, however, considers deterministic discrete-time linear systems. We extend our previous work by (i) considering stochastic and unknown system dynamics;
(ii) proving correctness of the proposed method and providing conditions under which the proposed algorithm is complete;
(iii) conducting numerical simulations and hardware experiments that do not exist in \cite{wang2022verified} demonstrating the sim2real gap. To the best of our knowledge, this is the first work that designs \textit{verified} compositions of NN controllers for systems with unknown and stochastic dynamics that yield system behaviors satisfying complex tasks captured by LTL.

\textbf{Contributions:} 
\textit{First}, we propose a new approach to design verified temporal compositions of trained NN controllers for systems with stochastic or even unknown dynamics and co-safe LTL control objectives.
\textit{Second}, we show correctness of the proposed algorithm and we discuss trade-offs between completeness and computational efficiency. 
\textit{Third}, we demonstrate the efficiency of the proposed approach on several case studies that involve complex navigation tasks for ground robots. 
\textit{Fourth}, we perform hardware experiments to investigate the benefits and limitations of the proposed method with respect to the sim2real gap, i.e., the difference between modeled and real performance.

%%%%%%%%%%%%%%%%%%%%%%%%%%%%%%%%%%%%%%%%%%%%%%%%%%%%%%%%%%%%%%%%%%%%%%%%%%%%%
\section{Problem Formulation}\label{sec:problem}

\subsection{Systems with Neural Network Controllers} 
We consider discrete-time nonlinear systems defined as:
\begin{equation}\label{eq:dynamics}
    \bbx_{t+1}=\bbf(\bbx_t, \bbu_t, \boldsymbol\nu_t) ,
\end{equation}
where $\bbx_t\in\ccalX\subseteq\mathbb{R}^d$,  $\bbu_t\in\ccalU_t\subseteq\mathbb{R}^n$ and $|\boldsymbol\nu_t| \leq V$ denote the state, the control input, and bounded exogenous disturbance of the system at time $t\geq0$.
%
% \textcolor{red}{[This is a common notational convention: use bold symbols only for vectors and not for scalars. So replace, $\bbN$ with $N$ and apply this change throughout the paper. $N$ has been used for the number of regions as well, so find another letter as well (maybe M?)]}
%
Notice that we do not assume the knowledge of the systems, i.e., $\bbf$ can be a simulator of the real system. %which we do not have access to the dynamics.
We consider $N>0$ sub-spaces of interest, denoted by $\ell_i \subset \ccalX$, that are inside $\ccalX$.
We assume that at any time $t$ the system can apply control inputs selected from a finite set of feedback controllers collected in the set $\Xi=\{\xi_{i}\}_{i=1}^N$, where $\xi_{i}(\bbx_t): \ccalX\to\mathbf{R}^n$ maps system states to control actions.
%
% \footnote{\textcolor{magenta}{this red text was missing. We said that we should apply any changes we made in the CDC paper here as well. Lots of them are missing throughout the paper though. Please read carefully your text and any comments I make on your drafts! }}
%
We consider controllers $\xi_{i}(\bbx_t)$ that are parameterized by multi-layer neural networks (NNs). %Such NN controllers can be implemented using available methods; see e.g., \cite{jothimurugan2021compositional,chen2022large}. 
We make the following assumption about the controllers $\xi_{i}(\bbx_t)$:

\begin{assumption}[NN Controllers]\label{as:NN1}
Every NN controller $\xi_{i}$ has been trained so that it can drive the system state $\bbx_t$ towards the interior of $\ell_i$ for any initial state in $\ccalX$.
\end{assumption}

We do not make any assumption about the performance and architecture of the trained NN controllers or how they have been trained. Hereafter, with slight abuse of notation, we denote by $\xi(t)$ the controller selected from $\Xi$ at time $t$.
To ensure that the NN output respects the input constraint, we consider a projection operator, denoted by $\text{Proj}_{\ccalU_t}$, and define the control input as $\bbu_t=\text{Proj}_{\ccalU_t}\xi(t)$ \cite{Hu2020ReachSDPRA}. 
%We denote the closed-loop system \eqref{eq:dynamics} and the projected NN control policy under bounded disturbance as:
%
%\begin{equation}\label{eq:closedloop}
%    \bbx_{t+1}=\bbf_{\xi}(\bbx_t, \boldsymbol\nu_t)
%\end{equation}
%
%\textcolor{red}{[we can probably remove that equation]}
Next, we define a NN-based control strategy $\boldsymbol\xi$ as a temporal composition of the controllers in $\Xi$.
\begin{definition}[Control Strategy]\label{def:NNcontrol}
A NN-based control strategy $\boldsymbol\xi$ is defined as a finite sequence of NN controllers selected from $\Xi$, i.e., $\boldsymbol\xi=\mu(0),\mu(1),\mu(2),\dots,\mu(K)$, for some finite $K>0$, where $\mu(k)\in\Xi$, for all $k\in\{0,1,\dots,K\}$, and $\mu(k)$ is applied for a finite horizon $H_k$. 
\end{definition} 
We highlight that in the above definition, the index $k$ is different from the time instants $t$ used in \eqref{eq:dynamics}. In fact, $k$ is an index, initialized as $k=0$ and increased by $1$ every $H_k$ time instants, pointing to the next NN controller in the sequence $\boldsymbol\xi$. Also, we note again that $\mu(k)$ is a feedback controller from $\Xi$. For instance, if $\mu(0)=\xi_i$, for some $i\in\{1,\dots,N\}$, then the system applies the controller $\xi_i(\bbx_t)$, for all time instants $t\in\{0,1,\dots,H_0\}$. 
Under a NN-based control strategy $\boldsymbol\xi$, an initial state $\bbx_0$, the corresponding closed-loop system %\eqref{eq:closedloop} 
can generate a finite sequence of system states, denoted by $\tau(\bbx_0)=\{\bbx_0, \bbx_1,\dots,\bbx_t,\dots, \bbx_F\}$, where $F=\sum_{k=0}^K H_k$ is a finite horizon. In fact, $\tau(\bbx_0)$ is a stochastic sequence due to the noise $\boldsymbol\nu_t$ in the system dynamics. % 

\subsection{Linear Temporal Logic Properties} We define mission and safety properties for the system \eqref{eq:dynamics} using Linear Temporal Logic (LTL). LTL is a type of formal logic whose basic ingredients are a set of atomic propositions (i.e., Boolean variables), denoted by $\mathcal{AP}$, the Boolean operators, i.e., conjunction $\wedge$, and negation $\neg$, and two temporal operators, next $\bigcirc$ and until $\mathcal{U}$. LTL formulas over a set $\mathcal{AP}$ can be constructed based on the following grammar: $\phi::=\text{true}~|~\pi~|~\phi_1\wedge\phi_2~|~\neg\phi~|~\bigcirc\phi~|~\phi_1~\mathcal{U}~\phi_2$, where $\pi\in\mathcal{AP}$. 
For brevity, we abstain from presenting the derivations of other Boolean and temporal operators, e.g., \textit{always} $\square$, \textit{eventually} $\lozenge$, \textit{implication} $\Rightarrow$, which can be found in \cite{baier2008principles}. 
Hereafter, we define the set $\mathcal{AP}$ as $\mathcal{AP}=\cup_i\{\pi^{\ell_i}\}$, where $\pi^{\ell_i}$ is an atomic predicate that is true when the system state $\bbx_t$ is within region $\ell_i$.
We restrict our attention to co-safe LTL properties that exclude the use of the `always' operator; see also Example \ref{ex:LTL}.
Co-safe LTL formulas are satisfied by finite sequences of states $\tau(\bbx_0)=\{\bbx_0, \bbx_1,\dots,\bbx_t,\dots, \bbx_F\}$ where $F$ is a finite horizon \cite{baier2008principles} that can be generated by \eqref{eq:dynamics} for some control strategy $\boldsymbol\xi$. Due to the noisy system dynamics, satisfaction of $\phi$ can only be reasoned probabilistically. Given a control strategy $\boldsymbol\xi$ and a set $\ccalX_0$ of initial states, hereafter, we denote by $\mathbb{P}_{\tau(\bbx_0)\sim D}(\tau(\bbx_0)\models\phi|\boldsymbol\xi,\ccalX_0)$ the probability that the corresponding closed-loop system %\eqref{eq:closedloop} 
will generate a sequence $\tau(\bbx_0)$ that satisfies $\phi$ for any initial state $\bbx_0\in\ccalX_0$, where $D$ is the distribution over an infinite number of trajectories generated by \eqref{eq:dynamics}. %\cite{baier2008principles}.
\begin{figure}[t]
\centering
\includegraphics[width=0.8\linewidth]{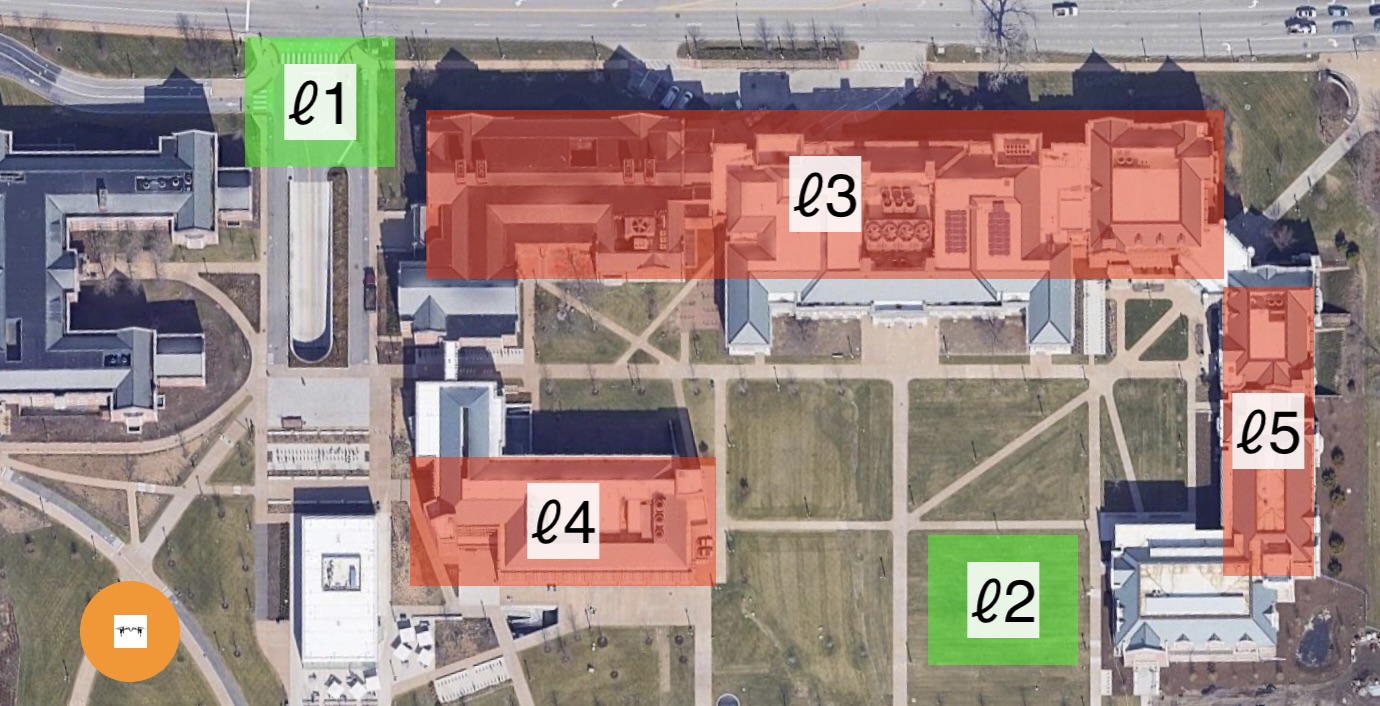}
\caption{%A top-view figure\cite{google_map_wustl}  
A drone is tasked with a delivery task over the campus of Washington University in St. Louis. This mission is captured by an LTL formula $\phi=(\Diamond \pi^{\ell_1}) \wedge (\Diamond \pi^{\ell_2}) \wedge (\neg \pi^{\ell_1} \ccalU \pi^{\ell_2}) \wedge (\neg (\pi^{\ell_3} \vee \pi^{\ell_4} \vee \pi^{\ell_5}) \ccalU \pi^{\ell_1})$ requiring to \textit{eventually}  reach region $\ell_2$ \textit{and} then \textit{eventually} reach region $\ell_1$ (green rectangles) \textit{in this order} to deliver packages. Meanwhile, the drone is also required to avoid flying over buildings located in $\ell_3$, $\ell_4$, and $\ell_5$ (red rectangles) \textit{until} all packages are delivered. The drone can start from a range of initial locations (orange disk) and is equipped with trained controllers $\xi_i$ for each region-to-reach $\ell_i$ (i.e., $\Xi = \{ \xi_1, \xi_2, \xi_3, \xi_4, \xi_5\}$). 
The drone is subject to bounded exogenous disturbances such as wind gusts. Our goal in this paper is to design verified temporal compositions of the NN controllers $\xi_i$ that satisfy the assigned LTL task with probability $1$. 
%
% {\textcolor{magenta}{I believe you are missing an arrow towards $\ell_4$. Also, to be consistent with the problem formulation, $\Xi$ should also include $\xi_3$, $\xi_4$, and $\xi_5$. Also, it would be helpful to draw a set of initial states. Finally, it might be a good idea to replace this figure with an actual top-view figure of a city  as e.g., in fig 14 on this paper: https://arxiv.org/pdf/2203.15661.pdf Then, we can mark which buildings/regions to avoid or reach. We can decide at the end how to update the figure (no action on this is needed now). }} 
} 
\label{fig:background}
\end{figure}
\begin{exmp}\label{ex:LTL}
Examples of co-safe LTL specifications $\phi$ follow: (i) $\phi=\Diamond(\pi^{\ell_1})\wedge (\neg \pi^{\ell_2}\ccalU \pi^{\ell_1})$ captures a common reach-avoid property requiring the system to eventually reach the region of interest $\ell_1$ while avoiding in the meantime the unsafe region $\ell_2$; (ii) $\phi=(\Diamond\pi^{\ell_1})\wedge (\Diamond\pi^{\ell_2}) \wedge (\Diamond\pi^{\ell_3})\wedge(\Diamond\pi^{\ell_4})\wedge (\neg \pi^{\ell_4} \ccalU \pi^{\ell_1})$, which requires the system to eventually reach the regions of interest $\ell_1$, $\ell_2$, $\ell_3$, and $\ell_4$, in any order, as long as $\ell_4$ is avoided until region $\ell_1$ is reached; 
(iii) $\phi=(\Diamond \pi^{\ell_1}) \wedge (\Diamond \pi^{\ell_2}) \wedge (\neg \pi^{\ell_1} \ccalU \pi^{\ell_2}) \wedge (\neg (\pi^{\ell_3} \vee \pi^{\ell_4} \vee \pi^{\ell_5}) \ccalU \pi^{\ell_1})$ , which requires the system to eventually reach regions $\ell_2$ and $\ell_1$ in this order, as long as $\ell_3, \ell_4, \ell_5$ are avoided until region $\ell_1$ is reached; see also Fig \ref{fig:background}.  % \textcolor{red}{[that isnt the formula you have in that fig although they model the same task. rewrite the formula so that it is the same in the fig caption and in the examples]}.
\end{exmp}
%\vspace{-0.4cm}

\subsection{Problem Statement}
%Given a simulator of system dynamics, and a specification $\phi$, our goal is to check if there exists - and if so, to compute it - a NN control strategy $\boldsymbol\xi$, as defined in Definition \ref{def:NNcontrol}, that can satisfy $\phi$ with probability $1$ for all initial states $\bbx_0\in\ccalX_0$, i.e., $\mathbb{P}_{\tau(\bbx_0)\sim D}(\tau(\bbx_0)\models\phi|\boldsymbol\xi,\ccalX_0)=1$; see Fig. \ref{fig:background}. 
The problem addressed in this paper can be summarized as follows:

\begin{problem}[Verification Problem]\label{pr1}
\textit{Given} (i) a set of initial states $\ccalX_0\subseteq\ccalX$; (ii) the system dynamics \eqref{eq:dynamics} (or a simulator); (iii) a co-safe LTL property $\phi$; (iv) a finite set $\Xi$ of NN controllers; and (v) disturbance bound $V$, check if there exists a NN-based control strategy $\boldsymbol\xi$ (Definition \ref{def:NNcontrol}), so that $\mathbb{P}_{\tau(\bbx_0)\sim D}(\tau(\bbx_0)\models\phi|\boldsymbol\xi,\ccalX_0)=1$, for all initial states $\bbx_0\in\ccalX_0$; if there exists such a $\boldsymbol\xi$, compute it, as well. 
\end{problem}

\section{Neuro-Symbolic Control \\for Temporal Logic Tasks}\label{sec:deeplogic}

We propose a new method to design verified compositions of NN controllers for co-safe LTL tasks; see Alg. \ref{algo1}.
First, we translate the LTL formula $\phi$ into a Deterministic Finite state Automaton (DFA) [line \ref{algo1:DFA}, Alg. \ref{algo1}]; see Section \ref{sec:DFA}. 
Second, building upon \cite{kantaros2020reactive,vasilopoulos2021reactive}, we leverage the DFA to decompose $\phi$ into reach-avoid properties [line \ref{algo1:dec}, Alg. \ref{algo1}]; see Sections \ref{sec:reachAvoid}-\ref{sec:VerReach}. Then, we apply graph-search methods combined with reachability analysis to check if there exists $\boldsymbol\xi$ so that 
$\mathbb{P}_{\tau(\bbx_0)\sim D}(\tau(\bbx_0)\models\phi|\boldsymbol\xi,\ccalX_0)=1$,
%
% $Pr_{\mathcal{D}_{(\boldsymbol\xi,\bbN)}}[\bbf_{\xi}\models\phi]=1$, 
%
% \textcolor{red}{[fix the notations]} 
%
for all $\bbx_0\in\ccalX_0$ [line \ref{algo1:reach}, Alg. \ref{algo1}]; see Section \ref{sec:vltl}. Then, the proposed algorithm returns either  \texttt{False} if such a control strategy $\boldsymbol\xi$ does not exist or, otherwise, \texttt{True} along with $\boldsymbol\xi$ [line \ref{algo1:output}, Alg. \ref{algo1}]. Trade-offs between completeness and computational efficiency are discussed in Section \ref{sec:complete}.

\begin{algorithm}[t]
%\footnotesize
\caption{Verified Compositional Neuro-Symbolic Control for Stochastic Systems with Temporal Logic Tasks}
\label{algo1}
\begin{algorithmic}[1]
\State \textbf{Input}: Formula $\phi$; System Dynamics/Simulator \eqref{eq:dynamics}; Controllers $\Xi$; Disturbance Bound $V$
\State Translate $\phi$ into DFA $D$ (Section \ref{sec:DFA})\;\label{algo1:DFA}
%\State Prune DFA $D$\; (see Section \ref{sec:prune})
\State Using $D$, decompose $\phi$ into reach-avoid properties (Sections \ref{sec:reachAvoid}-\ref{sec:VerReach})\;\label{algo1:dec}
\State Apply graph-search \& reachability analysis over the DFA state-space \label{a_dfs} (Section \ref{sec:vltl})\;\label{algo1:reach}
\State \textbf{Output}: Verification result  $R\in\{[\texttt{True},\boldsymbol\xi],\texttt{False}\}$\;\label{algo1:output}
\end{algorithmic}
\end{algorithm}
\normalsize

\subsection{From LTL formulas to DFA}\label{sec:DFA}
First, we translate  $\phi$, constructed using $\mathcal{AP}$, into a DFA defined as follows \cite{baier2008principles}; see also Fig. \ref{fig:example}.%
\begin{definition}[DFA]
A Deterministic Finite state Automaton (DFA) $D$ over $\Sigma=2^{\mathcal{AP}}$ is defined as a tuple $D=\left(\ccalQ_{D}, q_{D}^0,\Sigma,\delta_D,q_D^F\right)$, where $\ccalQ_{D}$ is the set of states, $q_{D}^0\in\ccalQ_{D}$ is the initial state, $\Sigma$ is an alphabet, $\delta_D:\ccalQ_D\times\Sigma\rightarrow\ccalQ_D$ is a deterministic transition relation, and $q_D^F\in\ccalQ_{D}$ is the accepting/final state. 
\end{definition}

\begin{figure}[t]
    \centering
    \begin{subfigure}[t]{0.48\linewidth}
        \includegraphics[width=\textwidth]{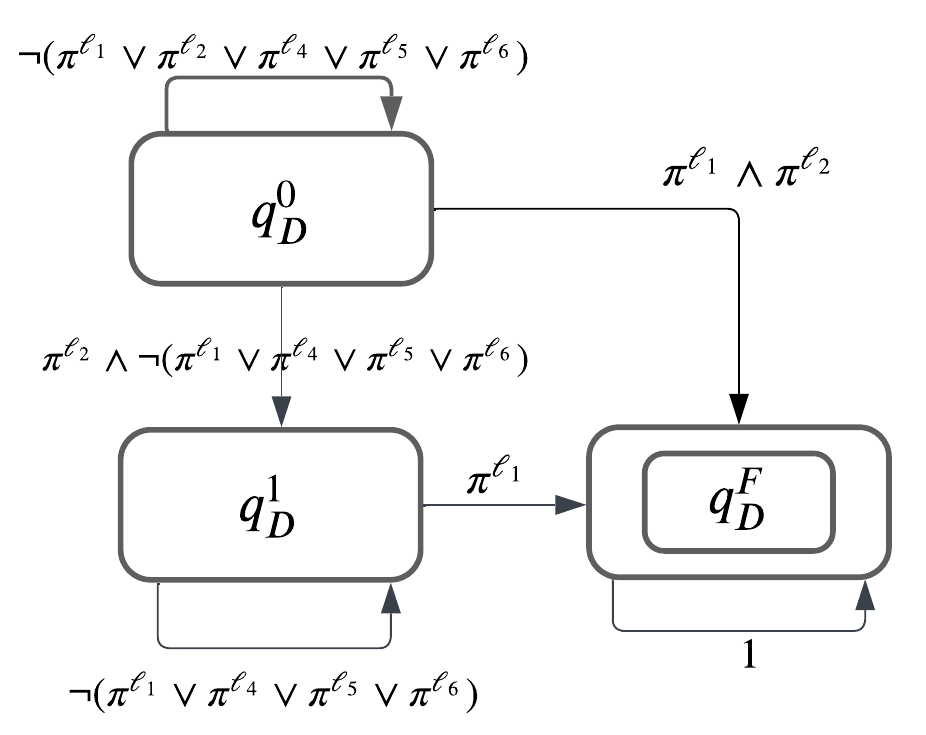}
        % \caption{DFA corresponding to $\phi=(\Diamond\pi^{\ell_1})\wedge(\Diamond\pi^{\ell_2}) \wedge (\neg \pi^{\ell_2} \ccalU \pi^{\ell_1})$}
        \caption{}
        \label{fig:example}
    \end{subfigure}
    \begin{subfigure}[t]{0.5\linewidth}
        \includegraphics[width=\textwidth]{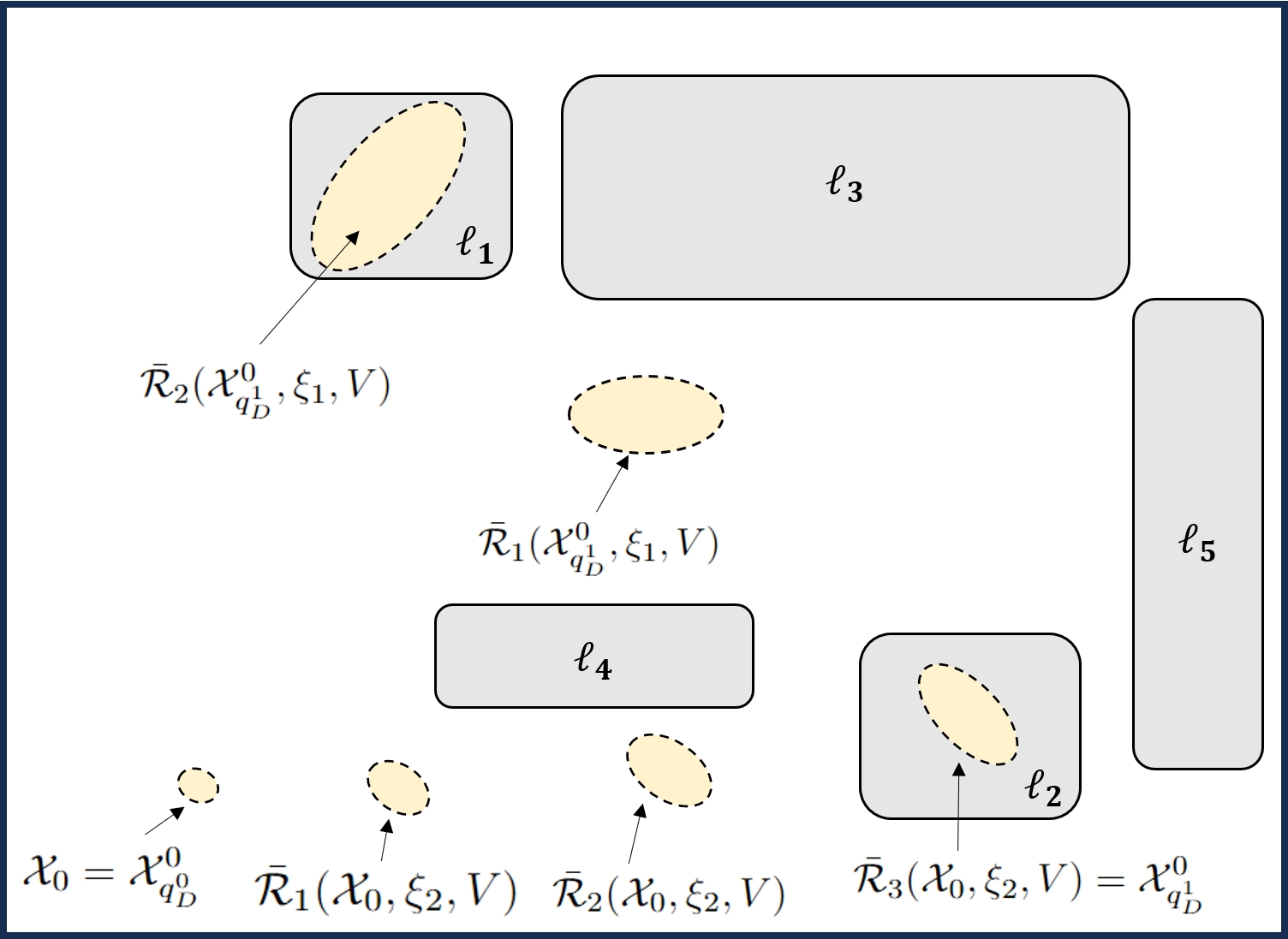}
        % \caption{Illustration of the reachable sets over the DFA state space (see Section \ref{sec:vltl})}
        \caption{}
        \label{rfidtest_yaxis}
    \end{subfigure}
    \caption{(a) DFA corresponding to $\phi=(\Diamond \pi^{\ell_1}) \wedge (\Diamond \pi^{\ell_2}) \wedge (\neg \pi^{\ell_1} \ccalU \pi^{\ell_2}) \wedge (\neg (\pi^{\ell_3} \vee \pi^{\ell_4} \vee \pi^{\ell_5}) \ccalU \pi^{\ell_1})$. % The red dashed edge corresponds to an infeasible transition as it requires the drone to be present in more than one location simultaneously \textcolor{red}{[Where are we talking about pruning infeasible transitions in the paper? I don't see this anywhere. It needs to appear in Section III-A if you want to talk about it in this caption. See also my comment at the beginning of Section III.B]}; 
    (b) Illustration of the reachable sets over the DFA state space (see Section \ref{sec:vltl}) %\textcolor{red}{[Change a bit figure 2b. maybe re-arrange the locations of the regions and the set of initial states. Or maybe plot the reachable sets over the WashU campus fig? (nice pic btw!) Also maybe update the formula? maybe add an obstacle both in the figure and in the formula? This may change a bit the examples which is good. We should try to minimize the overlap with the cdc paper as much as possible. Also, the notations in the reachable sets in the figure are not correct. ]}
    }
    \label{fig:reach}
\end{figure}

% \begin{figure}[htbp]
% \minipage{0.48\textwidth}
%   \includegraphics[width=\linewidth]{img/2.png}
%   \caption{DFA corresponding to $\phi=(\Diamond\pi^{\ell_1})\wedge(\Diamond\pi^{\ell_2}) \wedge (\neg \pi^{\ell_2} \ccalU \pi^{\ell_1})$ }
%   \label{fig:example}
% \endminipage\hfill
% \minipage{0.48\textwidth}
%   \includegraphics[width=\linewidth]{img/ExampleEnvironment.png}
%   \caption{Illustration of the reachable sets over the DFA state space (see Section \ref{sec:vltl}) \textcolor{magenta}{the notations are not consistent with the text. Must be fixed. Also, see my comment throughout the paper about the figs. They are not aligned. Either place them not side-by-side. Or if you want to place them side by side to save space, align them (See e.g., the subfig package)}}\label{fig:reach}
% \endminipage
% \end{figure}

To interpret a temporal logic formula over a realization of $\tau(\bbx_0)$ generated by \eqref{eq:dynamics}, we use a labeling function $L:\ccalX\rightarrow 2^{\mathcal{AP}}$ that maps system states to symbols $\sigma\in\Sigma$. 
A realization of $\tau(\bbx_0)$, denoted by $\bbx_0,\bbx_1,\dots, \bbx_F$  satisfies $\phi$ if the word $w=L(\bbx_0)L(\bbx_1)\dots L(\bbx_F)$ yields an accepting DFA run, i.e., if starting from the initial state $q_D^0$, each symbol/element in $w$ yields a DFA transition so that eventually the final state $q_D^F$ is reached \cite{baier2008principles}. 

\subsection{From DFA to Reach-Avoid Properties}\label{sec:reachAvoid}
Given any DFA state $q_D\in\ccalQ_D$, we compute a set $\ccalR_{q_D}$ that collects all DFA states that can be reached, in one hop, from $q_D$ using a symbol $\sigma$. This set is defined as: %\textcolor{red}{[is this based on the original or the pruned DFA? Be specific. If you want, you can avoid talking about pruning infeasible transitions altogether and instead keep the discussion very general as it already is (that's what i would recommend). In this case (i) talk about the pruning part in the simulations as in our CDC paper (ii) in Fig. 2a, you cannot talk about infeasible transitions.}
\begin{equation}\label{eq:R}
    \ccalR_{q_D}=\{q_D'~|~q_D'=\delta_D(q_D,\sigma), \sigma\in\Sigma\}.
\end{equation}
Then, given $q_D$ and for each  DFA state $q_D'\in\ccalR_{q_D}$, we introduce the following definitions. We construct a set that collects the states $\bbx\in\ccalX$, so that if the system state coincides with one of these states, then a symbol $\sigma=L(\bbx)\in\Sigma$ enabling this DFA transition will be generated. We collect these states in the set $\ccalX_{q_D\rightarrow q_D'}$, i.e.,
\begin{equation}\label{eq:X}
    \ccalX_{q_D\rightarrow q_D'} =\{\bbx\in\ccalX ~|~q_D'=\delta_D(q_D,L(\bbx))\}
\end{equation}

%In what follows, for simplicity, we assume that for all $q_D\in\ccalQ_D$ there exists a feasible self-loop around every DFA state $q_D$; later, in Section \ref{sec:VerReach}, we relax this assumption. 
Starting from any system state $\bbx\in\ccalX$ and a DFA state $q_D$, transition from $q_D$ to $q_D'\in\ccalR_{q_D}\setminus q_D$ will eventually occur after $H_{q_D}\geq0$ discrete time steps, if (i) the system state $\bbx_t$ remains within $\ccalX_{q_D\rightarrow q_D}$ for the next $H_{q_D}-1$ steps, and (ii) at time $t=H_{q_D}$, we have that $\bbx_{H_{q_D}}\in\ccalX_{q_D\rightarrow q_D'}$. Essentially, (i)-(ii) model a \textit{reach-avoid} requirement. 
For instance, (i) may require a robot to stay within the obstacle-free space and (ii) may require a robot to eventually enter a region.

\begin{exmp}[Reach-Avoid Properties]\label{ex2}
Consider the DFA shown in Fig. \ref{fig:example} that
corresponds to the LTL formula $\phi=(\Diamond \pi^{\ell_1}) \wedge (\Diamond \pi^{\ell_2}) \wedge (\neg \pi^{\ell_1} \ccalU \pi^{\ell_2}) \wedge (\neg (\pi^{\ell_3} \vee \pi^{\ell_4} \vee \pi^{\ell_5}) \ccalU \pi^{\ell_1})$. Since there exists a DFA transitions from $q_D^0$ to the states $q_D^0,q_D^1,q_D^F$, we have that $\ccalR_{q_D^0}=\{q_D^0,q_D^1,q_D^F\}$. Also, the Boolean formulas on top of the DFA transitions capture the conditions under which these transitions can be enabled. For instance, the self loop in $q_D^0$ is enabled if the robot is outside the regions \{$\ell_1$, $\ell_2$, $\ell_3$, $\ell_4$, $\ell_5$\}. Thus, we have that  $\ccalX_{q_D^0\rightarrow q_D^0}=\ccalX\setminus(\ell_1\cup\ell_2\cup\ell_3\cup\ell_4\cup\ell_5)$. Similarly, we have that $\ccalX_{q_D^0\rightarrow q_D^1}=%\ell_2$,
(\ccalX\setminus(\ell_1\cup\ell_3\cup\ell_4\cup\ell_5))\cap\ell_2$,
$\ccalX_{q_D^0\rightarrow q_D^F}=\ell_2\cap\ell_1$ and that $\ccalR_{q_D^1}=\{q_D^1,q_D^F\}$, $\ccalX_{q_D^1\rightarrow q_D^1}=\ccalX\setminus(\ell_1\cup\ell_3\cup\ell_4\cup\ell_5)$, and $\ccalX_{q_D^1\rightarrow q_D^F}=\ell_1$. 
\end{exmp}

\subsection{Verifying Reach-Avoid Properties}\label{sec:VerReach}
In what follows, we discuss how to check whether a transition from a DFA state $q_D$ to $q_D'\neq q_D$ can be enabled; later, we will discuss how this can be used to design verified control policies that satisfy the LTL formula with probability 1.
%
% \footnote{\textcolor{magenta}{We have said that we should not use the term `verify LTL properties', it is somewhat misleading. Replace with `to design verified control policies that satisfy the LTL formula with probability 1.'}} 
%
Specifically, we want to verify that given an initial set of system states associated with $q_D$, denoted by $\ccalX_{q_D}^0$, the  conditions (i)-(ii), discussed in Section \ref{sec:reachAvoid}, can be satisfied with probability 1; the detailed construction of $\ccalX_{q_D}^0$ will be discussed in Section \ref{sec:vltl}. 
Notice that $\ccalX_{q_D\rightarrow q_D'}$ may contain more than one region of interest $\ell_i$.
As mentioned in Assumption \ref{as:NN1}, for every region of interest $\ell_i$ in $\ccalX_{q_D\rightarrow q_D'}$, there exists a NN controller $\xi_i\in\Xi$ that has been trained to drive the system towards $\ell_i$. Hereafter, we collect all NN controllers associated with $\ccalX_{q_D\rightarrow q_D'}$ in a set denoted by $\Xi_{q_D\rightarrow q_D'}\subseteq\Xi$.\footnote{In case $\ccalX_{q_D\rightarrow q_D'}$ does not contain any region $\ell_i$, then $\Xi_{q_D\rightarrow q_D'}=\emptyset$ by definition of $\Xi$; see Ex. \ref{ex3}.}
Specifically, we want to show that there exists a finite horizon $H_{q_D}$ and at least one NN controller $\xi\in\Xi_{q_D\rightarrow q_D'}$, so that if the system evolves as per $\bbx_{t+1}=\bbf(\bbx_t,\xi, \boldsymbol\nu_t)$ then the following two conditions hold for all possible initial system states in $\ccalX_{q_D}^0$ and all possible noise values $\boldsymbol{\nu}_t, |\boldsymbol\nu_t|\leq V$: (i) $\bbx_{t} \in \ccalX_{q_D\rightarrow q_D}, \forall t\in[0,H_{q_D}-1]$ and (ii) $\bbx_{H_{q_D}}\in\ccalX_{q_D\rightarrow q_D'}$.
If such a horizon $H_{q_D}$ and controller $\xi\in\Xi_{q_D\rightarrow q_D'}$ exist, then by definition of the set $\ccalX_{q_D\rightarrow q_D'}$ in \eqref{eq:X}, we have that with probability one the transition from $q_D$ to $q_D$ (self-loop) is enabled  within the time interval $[0,H_{q_D}-1]$, and at the time step $t=H_{q_D}$ the transition from $q_D$ to $q_D'$ occurs. In this case, we say that the DFA transition from $q_D$ to $q_D'$ is \textit{verified to be safe} when the system starts anywhere within $\ccalX_{q_D}^0$ and applies the controller $\xi$.

To reason about safety of a DFA transition, we leverage existing reachability analysis tools that can compute forward reachable sets, denoted by $\ccalR_t(\ccalX_{q_D}^0,\xi,V)$, collecting all possible states $\bbx$ that the system may reach after applying a feedback NN controller $\xi$ for $t$ time steps
in the presence of noise, upper bounded by $V$, for any possible initial state in $\ccalX_{q_D}^0$. Given such reachable sets, it suffices to check if there exists a finite horizon $H_{q_D}$ and at least one controller $\xi\in\Xi_{q_D\rightarrow q_D'}$ such that the reachable sets satisfy the following two conditions: 
\begin{equation}\label{eq:condI}
    \ccalR_{t}(\ccalX_{q_D}^0,\xi,V)\subseteq \ccalX_{q_D\rightarrow q_D}, \forall {t}\in[0,\dots,H_{q_D}-1],
\end{equation}
and 
\begin{equation}\label{eq:condII}
    \ccalR_{H_{q_D}}(\ccalX_{q_D}^0,\xi,V)\subseteq \ccalX_{q_D\rightarrow q_D'}.
\end{equation}
If both conditions hold, we verify that the DFA transition from $q_D$ to $q_D'$ is \textit{safe} given the initial set of states $\ccalX_{q_D}^0$ and the controller $\xi$ \cite{huang2019reachnn}. 

Construction of exact reachable sets is computationally intractable.
Thus, instead, we compute over-approximated reachable sets, denoted by $\bar{\ccalR}_t(\ccalX_{q_D}^0,\xi,V)$, where $\ccalR_t(\ccalX_{q_D}^0,\xi,V) \subseteq \bar{\ccalR}_t(\ccalX_{q_D}^0,\xi,V)$,
%\textcolor{blue}{if these reachable sets over-approximate the true reachable sets},
%
using recently proposed reachability tools for stochastic systems \cite{lew2022simple}.
%
%\textcolor{magenta}{[you should cite the other paper too.]} 
%
If the following two conditions are satisfied
\begin{equation}\label{eq:reachi2}
\bar{\ccalR}_{t}(\ccalX_{q_D}^0, \xi,V)\subseteq \ccalX_{q_D\rightarrow q_D}, \forall {t}\in[0,\dots,H_{q_D}-1]
\end{equation}
\begin{equation}\label{eq:reachii2}
\bar{\ccalR}_{H_{q_D}}(\ccalX_{q_D}^0,\xi,V)\subseteq \ccalX_{q_D\rightarrow q_D'},
\end{equation}
then we say that the considered DFA transition is verified to be safe given the initial set of states $\ccalX_{q_D}^0$ and a feedback NN controller $\xi$.\footnote{In practice, reachable sets over a large enough horizon $\bar{H}$ are computed. If there is not reachable set $\bar{\ccalR}_t$, for some $t\in\{0,\dots,\bar{H}\}$ that satisfies \eqref{eq:reachii2}, then we say that the system fails to reach this region of interest. This is in accordance with related works; see e.g., \cite{Hu2020ReachSDPRA, huang2019reachnn}.} Finally, if there is no self-loop for $q_D$, then $\ccalX_{q_D\rightarrow q_D}$ cannot be defined. In this case, such a transition from $q_D$ to $q_D'$ is verified to be safe if \eqref{eq:reachii2} holds for $H_{q_D}=1$.

In what follows, we discuss how the over-approximated reachable sets are computed using the reachability analysis algorithm for stochastic systems proposed in \cite{lew2022simple}.
This algorithm estimates reachable sets in a sampling-based manner. %\textcolor{red}{[You need to describe this algorithm in a more general way, and not just for $t=0$. Also avoid very long sentences as much as possible. See my revision:]}
Consider a set of (initial) states $\ccalX_{q_D}^0$ that the system state may be in at $t=0$. Our goal is to compute the reachable sets at the next $H$ time steps, for some $H\geq 1$. To compute the reachable at time $t=1$, first, $M$ system states $\bbx_1$ are sampled from $\ccalX_{q_D}^0$. Then, using  a given NN controller $\xi$ the corresponding next states $\bbx_{1}$ are computed using a simulator of the system dynamics \eqref{eq:dynamics}.
Then, we compute a convex hull over all (simulated) system states  $\bbx_{1}$ and we apply padding of size $\epsilon$ to this convex hull. The padded convex hull is the estimated reachable set at $t=1$, denoted by $\bar{\ccalR}_{1}$. Using $\bar{\ccalR}_{t}$ as a set of possible initial states at time $t\geq 1$, we repeat the above steps to compute the subsequent reachable sets $\bar{\ccalR}_{t+1}$, for all $t\in\{2,\dots,H\}$. As per \cite{lew2022simple}, the probability that $\bar{\ccalR}_{t}$ over-approximates the corresponding true reachable set $\ccalR_{t}$ is at least $1-\delta_M$, where $\delta_M$ depends on system state dimension $d$,
the padding size $\epsilon$, and the number $M$ of sampled points, $\mathcal{P}(\ccalR_t \subseteq \bar{\ccalR}_t) \geq 1-\delta_M$.  %\text{ where } \delta_M = F(M, d, \epsilon)$.
Notice that the probability that $\ccalR_{t}\subseteq\bar{\ccalR}_{t}$ for all $t\in\{1, \dots,H\}$ is $\mathcal{P}_H = \prod_{t=1}^H\mathcal{P}(\ccalR_t \subseteq \bar{\ccalR}_t) \geq (1-\delta_M)^H$.

\begin{exmp}[Verifying Reach-Avoid Properties (cont)]\label{ex3}
Consider the DFA in Fig. \ref{fig:example} and, specifically, the transition from $q_D^0$ to $q_D^1$. To reach $q_D^1$ starting from $q_D^0$, the available controllers are $\Xi_{q_D^0\rightarrow q_D^1}=\{\xi_2\}$ by construction of the set $\ccalX_{q_D^0\rightarrow q_D^1}$; see Ex. \ref{ex2} and Section \ref{sec:problem}. Let $\ccalX_{q_D^0}^0=\ccalX_0$. Then, we compute reachable sets $\bar{\ccalR}_t(\ccalX_0,\xi_2,V)$. These reachable sets are shown in Fig. \ref{fig:reach}. Observe that $H_{q_D^0}=3$, since the sets $\bar{\ccalR}_t(\ccalX_0,\xi_2,V)$ for $t=0,1,2$ satisfy \eqref{eq:reachi2} and $\bar{\ccalR}_3(\ccalX_0,\xi_2,V)$ satisfies \eqref{eq:reachii2}. Thus, the transition from $q_D^0$ to $q_D^1$ is verified to be safe. For the transition from $q_D^1$ to $q_D^F$, we have that $\Xi_{q_D^1\rightarrow q_D^F}=\{\xi_1\}$. Verification of this transition will be discussed in Ex. \ref{ex4}. Finally, we have $\Xi_{q_D^0\rightarrow q_D^0}=\emptyset$ and $\Xi_{q_D^1\rightarrow q_D^1}=\{\xi_2\}$.
\end{exmp}

\subsection{Designing Verified Compositions of NN Controllers}\label{sec:vltl}

\begin{algorithm}[t]
%\footnotesize
\caption{Reach\_DFS Algorithm}\label{algo2}
\begin{algorithmic}[1]
\State \textbf{Input}: $D'$; NN controllers $\Xi$; System dynamics/Simulator \eqref{eq:dynamics}; Disturbance bound $V$
\State \textbf{Output}: Verification output  $R\in\{[\texttt{True},\boldsymbol\xi],\texttt{False}\}$
\State Initialize $q_D^{\text{cur}} \leftarrow q_D^0$; $\ccalX_{q_D^{\text{cur}}}^0=\ccalX_0$;  $\bbq=q_D^{\text{cur}}$; $\ccalV_{\text{vis}}=\{q_D^{\text{cur}}\}$;  $g(q_D^{\text{cur}})\leftarrow \ccalX_{q_D^{\text{cur}}}^0$;  $\boldsymbol\xi \leftarrow \emptyset$; $E=\texttt{False}$;
\While{$E \neq \texttt{True}$}
\State \text{Randomly select} $q_D^{\text{next}}\in \ccalR_{q_D^{\text{cur}}}$; $\ccalV_{\text{vis}} \leftarrow \ccalV_{\text{vis}} \cup q_D^{\text{next}}$
\State Compute set of available controllers $\Xi_{q_D^{\text{cur}}\rightarrow q_D^{\text{cur}}}$\; 
\If {$\exists$ $H_{q_D^{\text{cur}}}$ and $\xi\in\Xi_{q_D^{\text{cur}}\rightarrow q_D^{\text{next}}}$ for \eqref{eq:reachi2}-\eqref{eq:reachii2}}
\State $\bbq=\bbq|q_D^{\text{cur}}$; $\boldsymbol\xi\leftarrow \boldsymbol\xi| \xi$; $\ccalX_{q_D^{\text{next}}}^0\leftarrow \bar{\mathfrak{R}}_{H_{q_D^{\text{cur}}}}(q_D^{\text{cur}})$;
\State   $q_D^{\text{cur}} \leftarrow q_D^{\text{next}}$;
$g(q_D^{\text{next}})\leftarrow \ccalX_{q_D^{\text{next}}}^0$;
\If {$q_D^{\text{cur}}=q_D^F$}
\State $R = [\texttt{True},\boldsymbol\xi]$;  $E= \texttt{True}$;
\EndIf
\Else
\State $\ccalR_{q_D^{\text{cur}}} \leftarrow \ccalR_{q_D^{\text{cur}}} \setminus q_D^{\text{next}}$
\While {$\ccalR_{q_D^{\text{cur}}}\setminus \ccalV_{\text{vis}}  = \emptyset \wedge \bbq \neq \emptyset$}
\State $q_D^{\text{cur}} \leftarrow \bbq(\text{end})$; $\ccalX_{q_D^{\text{cur}}}^0\leftarrow g(q_D^{\text{cur}})$;
\State $\bbq\leftarrow \bbq \setminus \bbq(\text{end})$; $\boldsymbol\xi\leftarrow \boldsymbol\xi\setminus \boldsymbol\xi(\text{end})$
\EndWhile
\If{$\ccalR_{q_D^{\text{cur}}}\setminus \ccalV_{\text{vis}}  = \emptyset\wedge \bbq = \emptyset $}
\State $R\leftarrow \texttt{False}$; $E\leftarrow \texttt{True}$;
\EndIf
\EndIf
\EndWhile
\end{algorithmic}
\end{algorithm}
\normalsize

To check if there exists a composed NN control strategy $\boldsymbol\xi$ that satisfies $\phi$ with probability 1,
%
% \footnote{\textcolor{magenta}{Same comment as before. such statements are misleading and confuse the reviewers. Replace it with: `To check if there exists a composed NN control strategy $\boldsymbol\xi$' that satisfies $\phi$ with probability one,'. Same for the title. Replace it with `Designing Verified Compositions of Neural Network Controllers'}} 
%
it suffices to check that the final DFA state can be reached from the initial state by enabling a sequence of DFA transitions that are verified to be safe as shown in Section \ref{sec:DFA}.
To check this, first we view the DFA as a directed graph $D=\{\ccalV, \ccalE\}$ with vertices $\ccalV$ and edges $\ccalE$ that are determined by the set of states and transitions of the DFA. Then, by applying a depth-first search algorithm over the DFA, we check if it is possible to reach the final DFA state from the initial state by following a sequence of DFA transitions that are verified to be safe as shown in Section \ref{sec:VerReach}. In what follows, we discuss this process in detail.

As discussed in Section \ref{sec:VerReach}, to verify safety of a DFA transition, an initial set of states is needed denoted by $\ccalX_{q_D}^0$. This set captures all possible states in $\ccalX$ that the system may have when it reaches a DFA state $q_D$. As a result, $\ccalX_{q_D}^0$ depends on the previous DFA states that the system has gone through to reach $q_D$. To simplify the proposed algorithm, we pre-process $D$ so that each node in $D$ that can be reached through multiple paths (excluding self-loops) originating from $q_D^0$ is replicated so that each replica can be reached through a unique path (excluding self-loops).\footnote{Otherwise, $\ccalX_{q_D}^0$ should depend on both the state $q_D$ and the path to reach $q_D$ which may complicate the notations as well as the description of Alg. \ref{algo2}.} We pre-process the graph as follows. 
For each vertex $q_D\in \ccalV$, we define sets that collect its incoming and outgoing edges denoted by $\ccalE_{q_D}^{\text{in}}$ and $\ccalE_{q_D}^{\text{out}}$, respectively.
If the number of incoming edges (excluding self-loops) for $q_D$ is greater than $1$, i.e., $|\ccalE_{q_D}^{\text{in}}|>1$ we create $|\ccalE_{q_D}^{\text{in}}|$ copies of $q_D$ denoted by $q_D^i$. Each node $q_D^i$ has only one incoming edge which is selected to be the $i$-th edge in $\ccalE_{q_D}^{\text{in}}$, denoted by $\ccalE_{q_D}^{\text{in}}(i)$, while its outgoing edges remain the same as in the original node $q_D$. Then we add all copies to the graph and remove the original nodes $q_D$. We denote the resulting graph by $D'$. An example illustrating this step is shown in Fig. \ref{fig:pre_processing_dfa}.
%
% \footnote{\textcolor{magenta}{[You should provide a figure with such an example. For example, have a figure showing the DFA of Fig 9a and then show the new DFA D' as well.]}}

\begin{figure}[t]
    \centering
    \includegraphics[width=\linewidth]{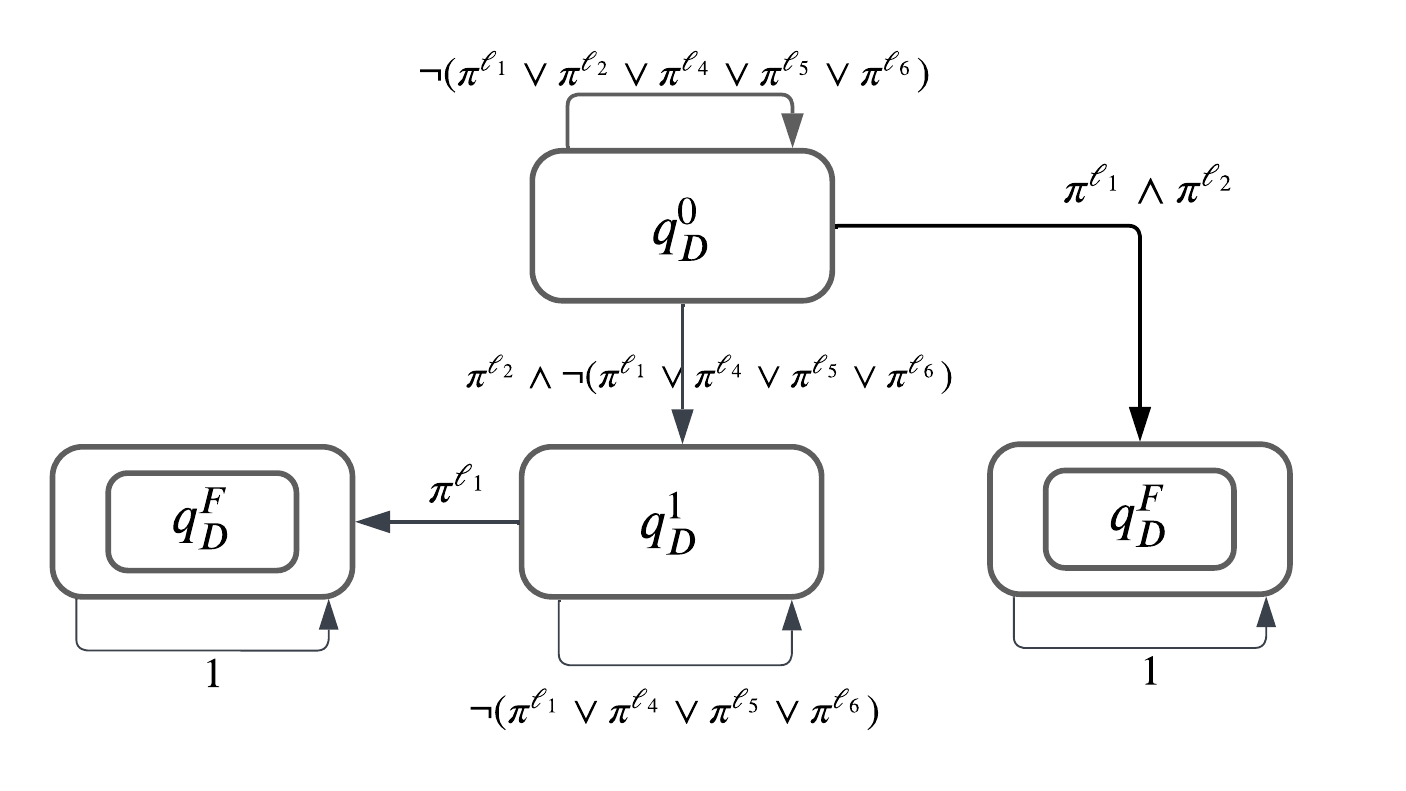}
    \caption{Illustration of the DFA pre-processing step, discussed in Section \ref{sec:vltl}, for the DFA shown in Fig. \ref{fig:example}. Notice in Fig. \ref{fig:example} that that there are two paths (excluding self-loops) to reach $q_D^F$ from $q_D^0$. %as the number of incoming edges (excluding self-loops) for $q_D^F$ is greater than 1
    }
    \label{fig:pre_processing_dfa}
\end{figure}

% \begin{figure}[htbp]
%     \centering
%     \begin{subfigure}[t]{0.4\linewidth}
%         \includegraphics[width=\textwidth]{img/case1_dfa.png}
%         \caption{}
%         \label{fig:case1_dfa_in_method}
%     \end{subfigure}
%     \begin{subfigure}[t]{0.55\linewidth}
%         \includegraphics[width=\textwidth]{img/case1_dfa_extend.png}
%         \caption{}
%     \end{subfigure}
%     \caption{(a) DFA corresponding to $\phi=(\Diamond\pi^{\ell_1})\wedge (\Diamond\pi^{\ell_2}) \wedge (\neg (\pi^{\ell_3} \vee \pi^{\ell_4} )\ccalU \pi^{\ell_1}) \wedge (\neg (\pi^{\ell_3} \vee \pi^{\ell_4} )\ccalU \pi^{\ell_2})$, red dotted edge connects $q_D^0$ and $q_D^F$ is \textit{pruned} as it requires the robot to be present in more than one location; (b) Illustration of the DFA pre-processing step as the number of incoming edges (excluding self-loops) for $q_D^F$ is greater than 1 (edge from $q_D^0$ to $q_D^F$ is pruned hence not shown here).}
%     \label{fig:pre_processing_dfa}
% \end{figure}

Next, we apply a depth-first search (DFS) method over  $D'$ to see if $q_D^F$ can be reached from $q_D^0$ through a sequence of DFA transitions that are verified to be safe. This process is summarized in Alg. \ref{algo2}. We note that the only difference with the standard DFS algorithm is that here reachability analysis is applied to check if a node $q_D$ can be reached from another node $q_D'$.
The inputs to Alg. \ref{algo2} are the graph $D'$, the set $\Xi$ of NN controllers and the system dynamics/simulator \eqref{eq:dynamics} [line 1, Alg. \ref{algo2}]. In what follows, we describe in detail how DFS is applied.

Let $q_D^{\text{cur}}$ denote the currently visited node in $D'$. We initialize $q_D^{\text{cur}}$ as $q_D^{\text{cur}}=q_D^0$. The set of all possible states in $\ccalX$ that the system can be when it reaches $q_D^{\text{cur}}$ is denoted by $\ccalX_{q_D^{\text{cur}}}^0$. We initialize $\ccalX_{q_D^{\text{cur}}}^0=\ccalX_0$. As required in the DFS algorithm, we define a set $\ccalV_{\text{vis}}$ that collects all nodes in $D'$ that have been visited and a sequence $\bbq$ of nodes that points to the current path from $q_D^0$ towards $q_F$. They are initialized as $\ccalV_{\text{vis}}=\{q_D^0\}$ and $\bbq=q_D^0$. Also we initialize the control strategy $\boldsymbol\xi$ as an empty sequence. We also define a function $g:\ccalV\rightarrow \ccalX$ that maps a DFA state $q_D\in\ccalV$ to the corresponding set $\ccalX_{q_D}^0$. This function that is constructed on-the-fly is needed only as way to store and recover from memory the sets $\ccalX_{q_D}^0$ [lines 2-3, Alg. \ref{algo2}]. 
Given  $q_D^{\text{cur}}$, we randomly select a next state $q_{\text{next}}\in \ccalR_{q_D^{\text{cur}}}$. Then we apply reachability analysis over the transition from $q_D^{\text{cur}}$ to $q_D^{\text{next}}$ using the NN controllers $\Xi_{q_D^{\text{cur}}\rightarrow q_D^{\text{next}}}\subseteq\Xi$ and the initial set of states $\ccalX_{q_D^{\text{cur}}}^0$ [lines 5-6, Alg. \ref{algo2}]; see Section \ref{sec:VerReach}. 
If the transition is verified to be safe, then we append $q_D^{\text{cur}}$ to $\bbq$. Also, we append the controller $\xi\in\Xi_{q_D^{\text{cur}}\rightarrow q_D^{\text{next}}}$ for which this transition is safe to $\boldsymbol\xi$ that denotes the current control strategy to reach $q_D^{\text{next}}$ from $q_D^0$ [lines 7-8, Alg. \ref{algo2}]. The corresponding horizon $H_{q_D^\text{cur}}$ should be stored as well, as it determines for how long the controller should be applied for; we abstain from this for simplicity of presentation.
The final reachable set $\bar{\ccalR}_{H_{q_D^{\text{cur}}}}(q_D^{\text{cur}},\xi, V)$ becomes the set $\ccalX_{q_D^{\text{next}}}^0$; see Ex. \ref{ex4} and Fig. \ref{fig:reach} as well [line 8, Alg. \ref{algo2}]. Also, $g(q_D^{\text{next}})$ is constructed on-the-fly as $g(q_D^{\text{next}})=\ccalX_{q_D^{\text{next}}}^0$ and we replace $q_D^{\text{cur}}$ with $q_D^{\text{next}}$ [line 9, Alg. \ref{algo2}].
If the transition $q_D^{\text{cur}}\rightarrow q_D^{\text{next}}$ is not verified to be safe, we remove $q_D^{\text{next}}$ from the set $\ccalR_{q_D^{\text{cur}}}$  (see \eqref{eq:R}). Then we remove the last element in $\bbq$ and $\boldsymbol\xi$, denoted by $\bbq(\text{end})$ and $\boldsymbol\xi(\text{end})$, while $\bbq(\text{end})$ is assigned to $q_D^{\text{cur}}$ until we find another state in $\ccalR_{q_D^{\text{cur}}}$ that has not been visited yet [lines 12-16, Alg. \ref{algo2}].
%
% \footnote{\textcolor{magenta}{Confirm that this step is correct and add a sentence explaining why they are removed. }} 
% because if we do not remove the q_cur, then the backtracking will be stuck on a certain state
%
The above process is repeated until $q_D^{\text{cur}}$ is updated to be $q_D^F$. In this case, we have found a NN control strategy $\boldsymbol\xi$ that satisfies $\phi$ for all initial states $\bbx_0\in\ccalX_0$ [line 10-11, Alg. \ref{algo2}]. If $\bbq$ is empty and we fail to find another state in $\ccalR_{q_D^{\text{cur}}}$ that is not visited, then the proposed method cannot find a feasible path from $q_D^0$ to $q_D^F$ (even though it may exist; see Section \ref{sec:complete}) [lines 17-18, Alg. \ref{algo2}].
Any other graph-search method in conjunction with reachability analysis can be used as well.

\begin{exmp}[Verified Compositional Control Synthesis
%
% \footnote{\textcolor{magenta}{Again, we should not use such terms. Replace it with `Verified Compositional Control Synthesis'}}
]
\label{ex4}
We continue Ex. \ref{ex3}; see also Fig. \ref{fig:reach}. To verify the DFA transition from $q_D^{\text{cur}}=q_D^1$ to $q_D^{\text{next}}=q_D^F$, we initialize $\ccalX_{q_D^1}^0=\bar{\ccalR}_3(\ccalX_0,\xi_2,V)$. 
By computing $\bar{\ccalR}_t(\ccalX_{q_D^1}^0,\xi_1,V)$, we verify that this DFA transition is safe. Thus, there exists $\boldsymbol\xi=\xi_2,\xi_1$ where $\xi_2$ and $\xi_1$ are applied for $3$ and $2$ time steps, respectively, so that 
$\mathbb{P}_{\tau(\bbx_0)\sim D}(\tau(\bbx_0)\models\phi|\boldsymbol\xi,\ccalX_0)=1$, 
%$Pr_{\mathcal{D}_{(\boldsymbol\xi,\bbN)}}[\bbf_{\xi}\models\phi]=1$, 
%
%\textcolor{red}{[fix the notation. please read the paper carefully and make sure we do not have such mistakes.]}
%
% \footnote{\textcolor{magenta}{you need to make sure the notations are consistent throughout the paper, otherwise the paper will be hard to read.}} 
%
for all $\bbx_0\in\ccalX_0$. 
\end{exmp}

%% how do we make our algorithm complete

\section{Correctness \& Completeness}\label{sec:complete}
In this section, first we show correctness of the proposed algorithm is probabilistically correct. %i.e., if Alg. \ref{algo1} generates a control strategy $\boldsymbol\xi$, then the stochastic system \eqref{eq:dynamics} satisfies the assigned LTL specification with probability $1$ \textcolor{black}{as long as the computed reachable sets over-approximate the true reachable sets}; see Proposition \ref{prop1}. 
Additionally, we provide conditions under which the proposed algorithm is complete and discuss trade-offs between completeness and computational efficiency.

% \textcolor{red}{[I changed a bit the proposition based on the remark you had put right after it.]}
%\subsection{Correctness}
\begin{prop}[Correctness]\label{prop1}
Algorithm \ref{algo1} is probabilistically correct, i.e., if it computes a NN control strategy $\boldsymbol\xi$ with horizon $F=\sum_{k=0}^KH_k$,  as defined in Definition \ref{def:NNcontrol}, then $\boldsymbol\xi$ satisfies $\phi$ with probability that is at least equal to $(1-\delta_M)^{F}$. 
\end{prop}

% \textcolor{red}{[the proof environment seems to not be defined - latex error]}
% \textcolor{red}{I made the blue edits very quickly. Read carefully and confirm that they are correct.}
\begin{proof}
This result holds by construction of Alg. \ref{algo1}. 
To show this, assume that Alg. \ref{algo2} returns a control strategy $\boldsymbol\xi$. Each feedback controller $\mu(k)$ in $\boldsymbol\xi$ is applied for $H_k\geq 1$ time-steps, i.e., $\mu(k)$ is applied for all time instants $t\in[t_k,t_k+H_k]$, where $t_0=0$ and $t_k=\sum_{n=0}^{k-1}H_{n}$ for all $k\geq 1$.
Consider the sequence of the computed reachable sets that correspond to $\boldsymbol\xi$, generated after applying every $\mu(k)$ in $\boldsymbol\xi$ for $H_k$ time units for $k=0,1,\dots,K$ in this specific order. By construction of the algorithm, these reachable sets satisfy \eqref{eq:reachi2}-\eqref{eq:reachii2}. 
\textcolor{black}{If the computed reachable sets over-approximate the true ones, then satisfaction of \eqref{eq:reachi2}-\eqref{eq:reachii2} implies that the conditions \eqref{eq:condI}-\eqref{eq:condII} hold as well.}
\textcolor{black}{The condition in \eqref{eq:condI}} implies that, given a fixed $k$, when the  system applies $\mu(k)$ for all time instants $t\in[t_k,t_k+H_k-1]$, it remains in a DFA state, denoted by $q_D^k$ with probability one for all $t\in[t_k,t_k+H_k-1]$ (when  $k=0$ this DFA state is the initial one). In other words, the system, driven by $\mu(k)$, generates observations (i.e., $L(\bbx_t)$) that allow it to stay in  $q_D^k$ for all $t\in[t_k,t_{k}+H_k-1]$. 
\textcolor{black}{The condition in \eqref{eq:condII}} implies that at time $t_{k+1}=t_k+H_k$, a new DFA state, denoted by $q_D^{k+1}$ will be reached with probability one.\footnote{Note that if $H_k=1$, this means that the system will stay at $q_D^k$ only at the time instant $t_k$. This can be the case e.g.,  because there is no self-loop at $q_D^k$ or because the self-loop transition at $q_D^k$ is not activated once the system reaches this DFA state.} The above holds for all $k=0,1,\dots,K$, and, therefore, a new DFA state $q_D^k$ is reached with probability one at every time instant $t_k$. This results in a sequence of DFA states $\bbq=q_D^0,q_D^1,\dots,q_D^k,\dots,q_D^K$.
By construction of Alg. \ref{algo2}, a policy $\boldsymbol\xi$ is returned when the final DFA state is included in $\bbq$. Thus, at time instant $t_K$ the DFA state $q_D^K$ that is reached is the final one. As a result, the resulting policy ensures that if the system executes the control strategy $\boldsymbol\xi$, then a sequence of DFA transitions leading to the final state will be activated with probability one only if the computed reachable sets over-approximate the true reachable sets. Since the latter holds with probability of at least $(1-\delta_M)^F$ (see Sec.\ref{sec:VerReach}), this equivalently means that the LTL formula is satisfied with probability that is at least equal to $(1-\delta_M)^F$ completing the proof.
\end{proof}

\begin{remark}[Computational Efficiency vs Completeness]\label{rem:complete}
In general, our method is not complete in the sense that it may not find a control strategy $\boldsymbol\xi$ that satisfies $\phi$ with probability one, for all $\bbx_0\in\ccalX_0$, even though such a strategy exists. This is due to the fact that (a) Alg. \ref{algo2} computes over-approximated (instead of exact) reachable sets and that (b) it does not exhaustively search over all possible combinations of NN control actions that the system can apply when  a new DFA state is reached. 
To elaborate more on (b), consider the following cases. (b.i) \textit{Given an initial set of states for $q_D^{\text{cur}}$}, Alg. \ref{algo2} reasons about safety of a transition from $q_D^{\text{cur}}$ to $q_D^{\text{next}}$, using only controllers selected from $\Xi_{q_D^{\text{cur}}\rightarrow q_D^{\text{next}}}$. If this transition is unsafe, then it is discarded. However, this transition may become feasible if the initial set of states changes, which can happen by applying a controller from $\Xi_{q_D^{\text{cur}}\rightarrow q_D^{\text{cur}}}$ for some $\hat{H}$ time steps. Searching over all possible values of $\hat{H}$ that may result in safe transitions is computationally intractable as $\hat{H}$ is not upper bounded. This can possibly be alleviated by considering bounded LTL formulas \cite{latvala2004simple}.
(b.ii) Additionally, as soon as Alg. \ref{algo2} finds a $\xi\in\Xi_{q_D^{\text{cur}}\rightarrow q_D^{\text{next}}}$ for which the corresponding transition is safe, it proceeds to new DFA transitions. However, that transition may be safe for other controllers in $\xi\in\Xi_{q_D^{\text{cur}}\rightarrow q_D^{\text{next}}}$ as well, where each one yields a different initial set for subsequent DFA transitions affecting their safety. This can be tackled by exhaustively searching over all possible controllers that the system can apply to enable the transition from $q_D^{\text{cur}}$ to $q_D^{\text{next}}$.
Addressing the challenges discussed in (b.i)-(b.ii) will increase the computational cost of Alg. \ref{algo2}. 
Similar trade-offs between completeness and computational efficiency are quite common in related works; see e.g., \cite{leahy2022fast}.
The proposed method in its current form is complete if (i) there are no self-loops in the DFA states, or if $\Xi_{q_D\rightarrow q_D}=\emptyset$ for all states (see e.g., Ex. \ref{ex3}); (ii) $|\Xi_{q_D\rightarrow q_D'}|=1$, for all $q_D\in\ccalQ_D\setminus\{q_D^F\}$; (iii) the reachable sets are accurately computed. Conditions (i)-(ii) trivially address (b.i)-(b.ii) while (iii) trivially tackles (a).
\end{remark}

\section{Experiments}\label{sec:sim}

%In this section, we demonstrate our method on complex navigation tasks for ground robots. Specifically, first, in Section \ref{sec:dymcon}, we define the dynamics of the considered robots and we discuss how we design the sets $\Xi$ of NN controllers. Second, in Section \ref{sec:grobot} we illustrate the performance of the proposed algorithm on various numerical experiments for a ground robot. Third, in Section \ref{sec:hware_exp}, we present hardware experiments that demonstrate the benefits and limitations of our approach with respect to the sim2real gap. In our experiments, we employ $\epsilon$-RandUP \cite{lew2022simple}, a recently proposed reachability analysis tool that handles systems with unknown dynamics. All simulations have been run on a computer with 64 GB RAM and a GeForce RTX 3080 graphic card.
In this section, we demonstrate our method on complex navigation tasks for ground robots. In Section \ref{sec:dymcon}, we define the dynamics of the considered robots and we discuss how we design the sets $\Xi$ of NN controllers. Using this set of base controllers, in Sections \ref{sec:grobot}-Section \ref{sec:hware_exp}, we design new composite controllers that can accomplish more complex tasks, with temporal and logical requirements, in a zero-shot fashion. 
Particularly, in Section \ref{sec:grobot} we illustrate the performance of the proposed algorithm in simulated settings. Next, in Section \ref{sec:hware_exp}, we present hardware experiments that demonstrate the benefits and limitations of our approach with respect to the sim2real gap. %In our experiments, we employ $\epsilon$-RandUP \cite{lew2022simple}, a recently proposed reachability analysis tool that handles systems with unknown dynamics.
All simulations have been run on a computer with 64 GB RAM and a GeForce RTX 3080 graphic card.

\subsection{Experiment Setup: Robot Dynamics, NN Control Design, and Reachability Analysis}\label{sec:dymcon}

\noindent
\textbf{Ground Robot Dynamics:}
We consider a ground robot with the following  differential drive
dynamics:
\begin{equation}\label{eq:grob}
\begin{bmatrix}x_{t+1}^1\\x_{t+1}^2\\ \theta_{t+1}\end{bmatrix}
=
\begin{bmatrix}x_t^1\\x_t^2\\ \theta_t\end{bmatrix}
+
\begin{bmatrix}
u_1 \sinc(\frac{u_2\tau}{2})\cos(\theta_t+\frac{u_2\tau}{2})\\
u_1 \sinc(\frac{u_2\tau}{2})\sin(\theta_t+\frac{u_2\tau}{2})\\
\tau u_2
\end{bmatrix}
+
\begin{bmatrix}\nu_1\\ \nu_2\\ \nu_3 \end{bmatrix}
\end{equation}

The ground robot state is defined as $\bbx=[x_t^1;x_t^2;\theta_t]\in \mathbb{R}^3$ capturing the position and direction of the robot at time $t$. The vector control input $\bbu\in \mathbb{R}^2$ consists of the linear velocity $u_1$ and angular velocity $u_2$, where $u_1\in [-0.22,0.22]$ and $u_2 \in [-0.15,0.15]$ for any $t>0$. %\textcolor{red}{[there is no $u_1$ and $u_2$ in \eqref{eq:grob}- please fix the notations!]} 
The noise signal $\boldsymbol\nu=[\nu_1; \nu_2; \nu_3] \in \mathbb{R}^3$ is bounded such that $\forall i \in \{1,\dots,3\}, \nu_i \in [-0.002, 0.002]$.% \textcolor{red}{[avoid bold letters for scalar variables - it should be $\nu_i$ instead of $\boldsymbol\nu_i$. apply this change throughout the paper]} 

\noindent
\textbf{NN Control Design:}
Next, we discuss how we train the NN controllers for the ground robot using Model Predictive Control (MPC) methods. As it will be discussed later, the LTL tasks require the robot to reach/avoid disjoint regions interests $\ell_i\in\mathbb{R}^2$ that belong to the physical workspace. %Specifically, for the ground robot case studies, the regions of interest $\ell_i\in\mathbb{R}^2$ are defined over the 2D robot position. 
To train a NN controller $\xi_i$, first we generate a random set of states $\bbx\in\mathbb{R}^3$. Starting from each one of these states, we generate a sequence of pairs of states and control inputs that drive the robot towards the interior of $\ell_i$ using an off-the-shelf MPC solver \cite{gill2005snopt}. Using the resulting training dataset, we train a feedforward NN controller $\xi_i$ with $2$ hidden layers, $40$ neurons per layer and ReLU activation functions that imitates  MPC. The performance of trained NN controllers can be found in Appendix \ref{sec:appendix_detail}. % \textcolor{red}{[if it is in Appendix, then move the following discussion in Appendix. otherwise bring back the table]} 

\vspace{-0.1cm}
\subsection{Numerical Simulations}\label{sec:grobot}
\vspace{-0.1cm}

\begin{figure}[t]
    \centering
    \begin{subfigure}[b]{0.5\linewidth}
        \includegraphics[width=\textwidth]{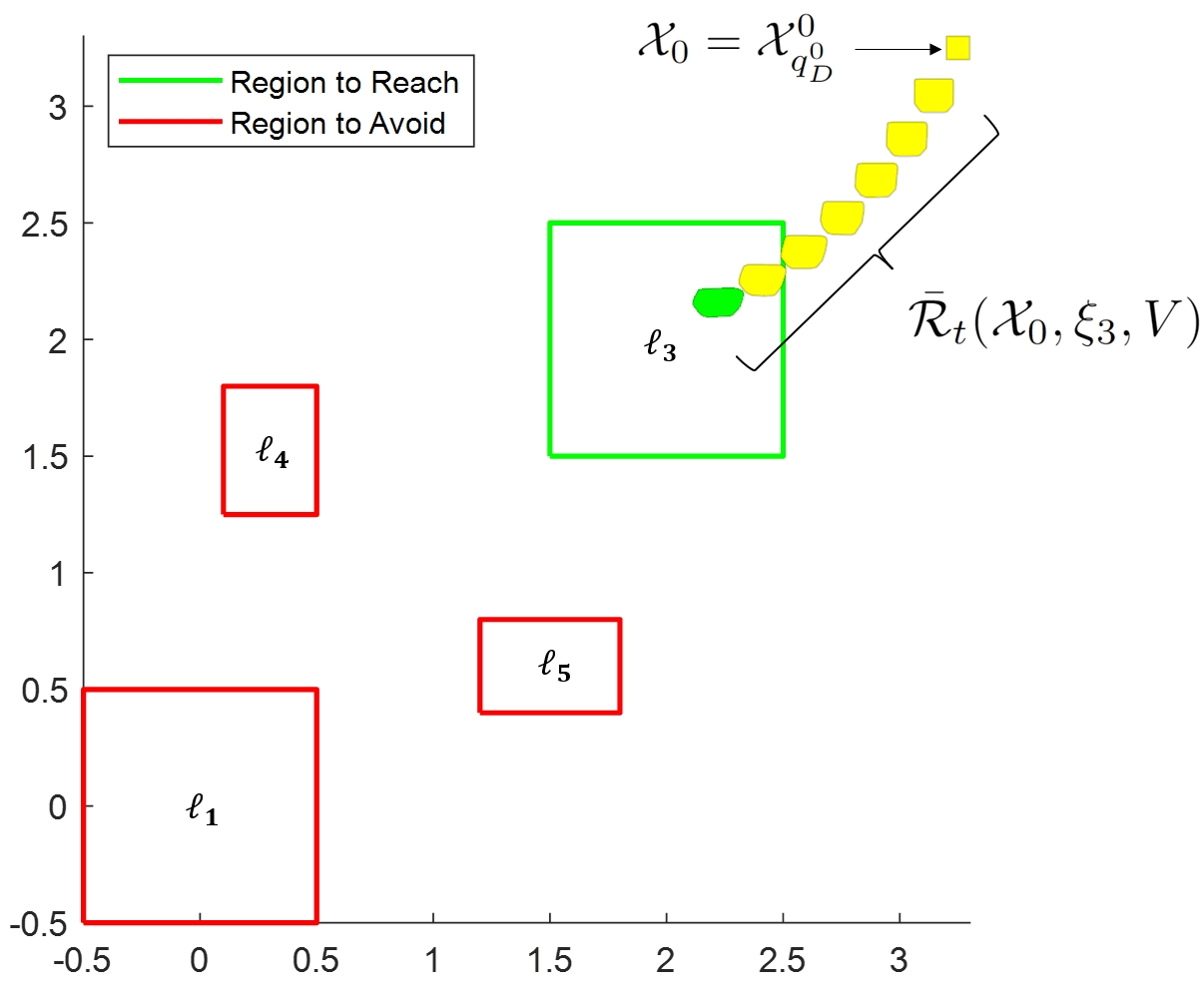}
        \caption{$q_D^0\rightarrow q_D^1$ ($\epsilon=0.03$)}
        \label{fig:g1suc}
    \end{subfigure}
    \begin{subfigure}[b]{0.45\linewidth}
        \includegraphics[width=\textwidth]{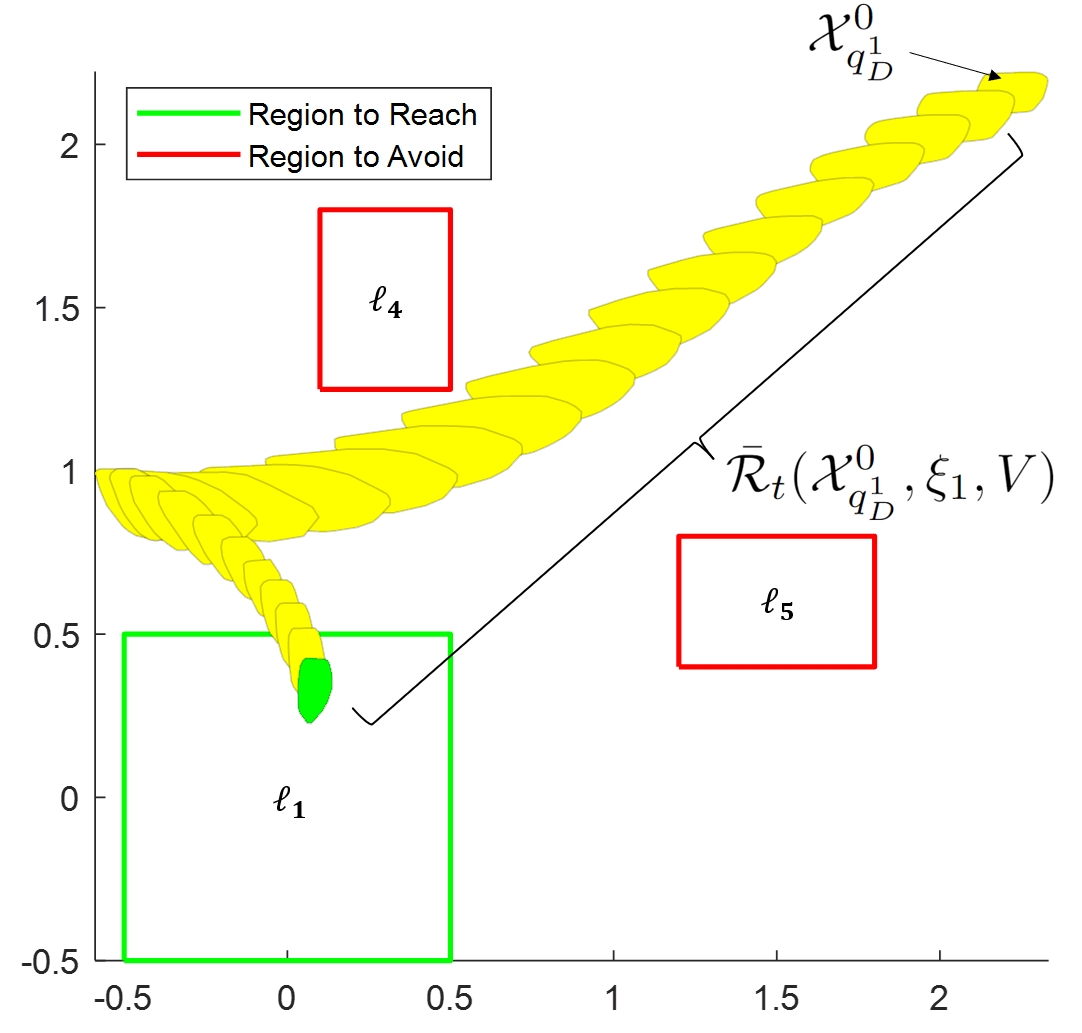}
        \caption{$q_D^1\rightarrow q_D^F$ ($\epsilon=0.03$)}
        \label{fig:g1suc2}
    \end{subfigure}
    \caption{Verification of DFA transitions for Section \ref{sec:grobot} ($M=600,000$) . % \textcolor{red}{[what is `Case I'?]}
The regions of interest that need to be avoided/reached to enable this DFA transition are shown with red/green color. All other regions are not shown. All reachable sets that are outside/inside the regions to reach are shown with yellow/green color. % \textcolor{red}{I suggest you put the region labels in the other figs too. alSO, these figs are really hard to read. Increase the fontsize if possible. Or otherwise, split them in two figures. }
}\label{fig:NumReach}
\end{figure}

\begin{figure}[t]
    \centering
    \begin{subfigure}[b]{0.53\linewidth}
        \includegraphics[width=\textwidth]{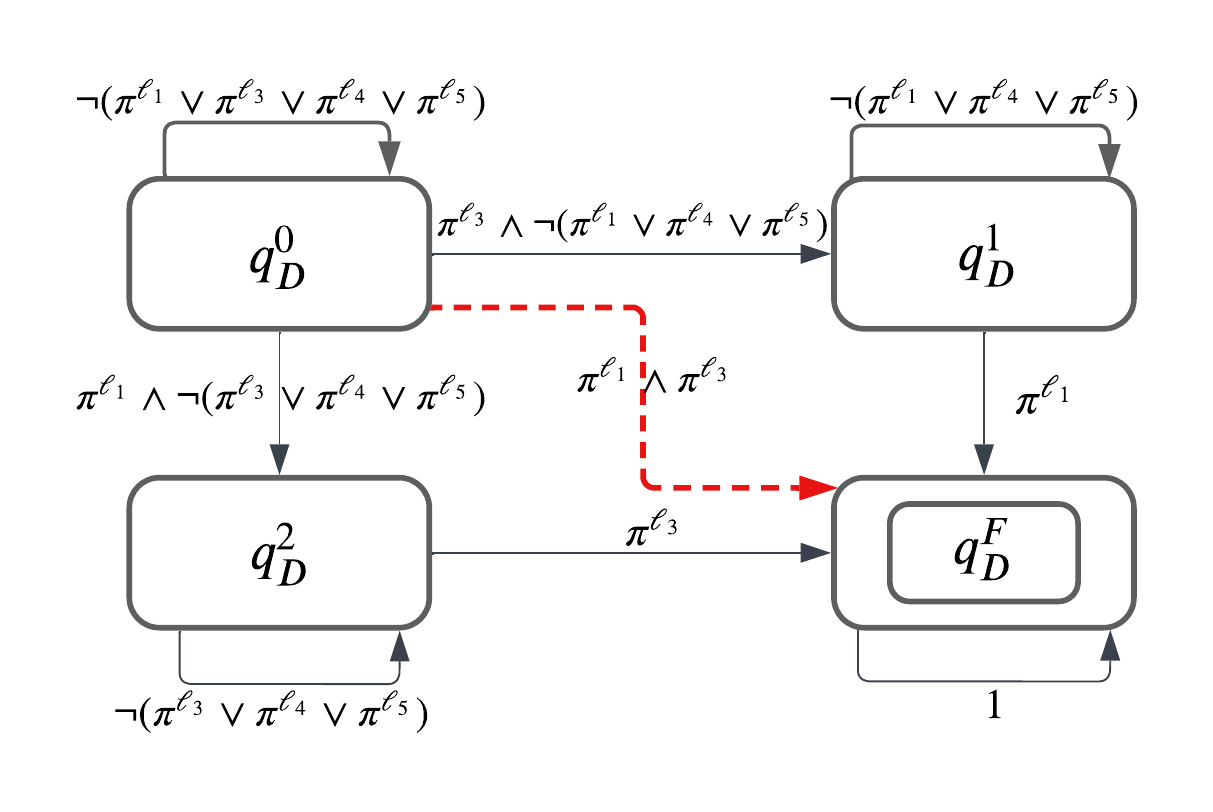}
        \caption{}
        \label{fig:g1first}
    \end{subfigure}
    \begin{subfigure}[b]{0.44\linewidth}
        \includegraphics[width=\textwidth]{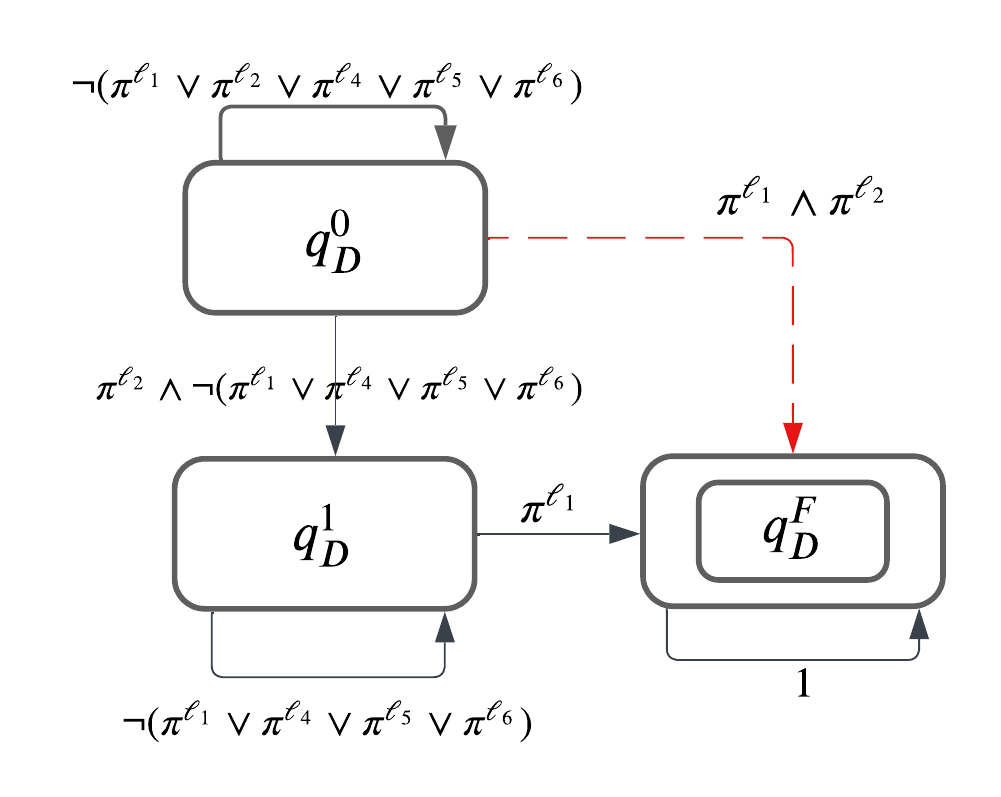}
        \caption{}
        \label{fig:hard_dfa}
    \end{subfigure}
    \caption{(a) DFA for the numerical experiment; (b) DFA for hardware experiment. The red dashed transition corresponds to an infeasible transition. }\vspace{-0.4cm}
\end{figure}

In what follows, we demonstrate the proposed algorithm in an LTL-encoded navigation task. Additional simulations can be found in Appendix \ref{sec:appendix_extra}. Specifically, we consider a ground robot residing in an environment shown in Fig. \ref{fig:NumReach} and % \textcolor{red}{[I think we had one more case study with the ground robot - put it in Appendix since we have space there. A single case study might not be enough for reviewers]}
%
%formula with complex task requirements (e.g., various numbers of obstacles or requirements in what order regions of interest should be visited).
% In what follows, we demonstrate the proposed algorithm over LTL formulas with different task complexity (e.g., various numbers of obstacles or requirements in what order regions of interest should be visited). % \textcolor{red}{[Where are those case studies of increasing task complexity? I see only one case study. If we dont have more cases, then this paragraph needs to be rephrased and also there is no need to call the next paragraph `case study I'. If we have more case studies, it would be good to include one more here (but the ones with the UAV).]}
%
tasked with the following LTL formula 
\begin{equation}\label{eq:LTL1}
    \phi=(\Diamond\pi^{\ell_1})\wedge (\Diamond\pi^{\ell_3}) \wedge (\neg (\pi^{\ell_4} \vee \pi^{\ell_5} )\ccalU \pi^{\ell_1}) \wedge (\neg (\pi^{\ell_4} \vee \pi^{\ell_5} )\ccalU \pi^{\ell_3})\nonumber.
\end{equation}
This task requires the robot to eventually visit the regions $\ell_1$ and $\ell_3$, in any order, while avoiding the obstacles $\ell_4$ and $\ell_5$. %(see Fig. \ref{fig:g1suc}).
This LTL formula corresponds to a DFA with 4 states and 5 transitions as shown in Fig. \ref{fig:g1first}. Notice that the transition from $q_D^0$ to $q_D^{F}$ is pruned and, therefore, not considered during verification, since it requires the robot to be present in more than one region simultaneously \cite{kantaros2020stylus}. % \textcolor{red}{[Let's bring the automata figures in the main text. It feels strange to refer to DFA states and transitions throughout this section without having the figures here. Also, figures and algorithms should be on the top of the page [t]]}
%\textcolor{red}{[Move the sentence about pruning to the appendix. Dont discuss here figures that do not appear in the main paper]}
%
The initial set $\ccalX_0$ is defined as $\{ \bbx_0 \in \ccalX | x_0^1, x_0^2 \in [3.2, 3.3], \theta_0 \in [4.24, 4.26] \}$.
% The initial set $\ccalX_0$ of system states is defined as a square ranging from $[3.2;3.2]$ to $[3.3;3.3]$. % \textcolor{red}{[something is missing from this sentence; does not read well]}. 
%
To compute the reachable sets, we apply $\epsilon$-RandUP using $M=600,000$ sampled points for the initial set of states with padding size of $\epsilon=0.03$. This guarantees that each reachable set at $t\geq 1$ is the outer-approximation of its true reachable set with a probability greater than $0.9996$ given $\ccalX_0$ \cite{lew2022simple}.  % \textcolor{red}{[is this true regardless of the DFA transition and the set of initial states?]}
Our method, first, investigates the DFA transition from $q_D^0$ to $q_D^1$.
%
% \footnote{\textcolor{magenta}{I think you need to also mention what is the self-loop transitions at $q_D^0$ and that during 1-7 secs, the robot stays at $q_D^0$. }} 
%
This transition requires the ground robot to stay within the obstacle-free space (i.e., avoid the obstacles $\ell_1$, $\ell_4$ and $\ell_5$) and eventually reach $\ell_3$. \textcolor{black}{Therefore, we have that $\Xi_{q_D^0\rightarrow q_D^1} = \{\xi_3\}$.} The corresponding reachable sets for this DFA transition, constructed using the controller $\xi_3$, are shown in Figure \ref{fig:g1suc}. Observe that the reachable sets $\bar{\ccalR}_1(\ccalX_0,\xi_3, V),\dots, \bar{\ccalR}_6(\ccalX_0,\xi_3, V)$ are fully outside the obstacle regions (as required by the self-loop transition of $q_D^0$) while $\bar{\ccalR}_7(\ccalX_0,\xi_3, V)$ is fully inside $\ell_3$. Thus, the transition for $q_D^0$ to $q_D^1$ is verified to be safe with probability that is at least equal to $0.9996^7$. 
Once the DFA state $q_D^1$ is reached, its self-loop transition is satisfied since the system state is outside $\ell_1$, $\ell_4$ and $\ell_5$. % \textcolor{red}{[it's hard for someone to understand such statements without the DFA figs. let's bring them back. ]} %until the transition $q_D^1$ to $q_D^F$ is satisfied.
Next, the DFA transition from $q_D^1$ to $q_D^F$ is considered requiring the robot to reach $\ell_1$ while avoiding the obstacle regions. \textcolor{black}{Thus, in this case, we have that $\Xi_{q_D^1\rightarrow q_D^F} = \{\xi_1\}$.}
The set of initial states for this transition is $\ccalX_{q_D^1}^0=\bar{\ccalR}_7(\ccalX_0,\xi_3, V)$. %that will be used for the subsequent reachability analysis. % \textcolor{red}{[or $\bar{\ccalR}_7$?]}.
%
%\textcolor{red}{fix the notations in all reachable sets.}. 
%
We again sample $M=600,000$ points within $\ccalX_{q_D^1}^0$ in our reachability analysis. % \textcolor{red}{[I don't see where this is coming from. Where did the power of 2 come from?]}  
%
% \textcolor{red}{[Be a bit more specific here. Did you sample again 500,000 points?]}
After computing the reachable sets $\bar{\ccalR}_t(\ccalX_{q_D^1}^0,\xi_1, V)$ shown in Figure \ref{fig:g1suc2}, we can see that the reachable sets over next $25$ time steps are outside the obstacles while $\bar{\ccalR}_{26}(\ccalX_{q_D^1}^0,\xi_1, V)$ is inside $\ell_1$ verifying that this DFA transition is also safe with probability at least equal to $(0.9996)^{26}$. Thus, there exists a control strategy $\boldsymbol\xi=\xi_3,\xi_1$ where $H_1=7,H_2=26$ such that 
% $Pr_{\mathcal{D}_{(\boldsymbol\xi,\bbN)}}[\bbf_{\xi}\models\phi]=0.9681$.
$\mathbb{P}_{\tau(\bbx_0)\sim D}(\tau(\bbx_0)\models\phi|\boldsymbol\xi,\ccalX_0)\geq(0.9996)^{33}=0.9869$.
%
% \textcolor{red}{fix the notations. I had the same comment before. Please read both my comments and the paper carefully, otherwise we will keep sending this back and forth.}% $\bbf_{\xi}\models\phi$ for all $\bbx_0\in\ccalX_0$ with the probability of $0.9681$. % \textcolor{red}{[What is this notation $\bbf_{\xi}\models\phi$? I dont think it is anywhere introduced. It is your responsibility to ensure that all notations are consistent throughout the paper.]}

% \textcolor{red}{[I rephrased a bit here the discussion on how we select the parameters $M$ etc. Make sure that it's correct.]}
Next, we discuss how the conservativeness of reachability analysis (arising due to overapproximation of reachable sets) can affect control design. To this end, we decrease the number of sampled points from $M=600,000$ to $M=50,000$. 
Using $M=50,000$, $\epsilon$-RandUP requires the padding size $\epsilon$ to increase from $0.03$ to $0.1$ to ensure that the generated reachable sets over-approximate the true ones with probability greater than $0.9999$. This essentially results in larger reachable sets which may affect the output of Algorithm \ref{algo1}. Specifically, using $\epsilon=0.1$, the transition from $q_D^0$ to $q_D^1$ is still verified to be safe. However, the transition from $q_D^1$ to $q_D^F$ is verified to be unsafe as the reachable sets intersect with the obstacles. The latter reachable sets are illustrated in Figure \ref{fig:g1los} in Appendix \ref{sec:appendix_detail}.
Notice that decreasing the number of sampled points can decrease the total runtime required for reachability analysis. Specifically,
the runtime to compute all reachable sets for verification 
%
% \textcolor{red}{[do you remember if this runtime refers to computing a single reachable set or to computing all reachable sets for verification?]} 
%
when $M=50,000$ and $M=600,000$ is $12$ and $66$ seconds, respectively.
%; while using $M=600,000$ will require $66$ seconds to finish the whole process.
%
% \textcolor{red}{What was the runtime before with the larger M? It was reported on a table} 
%
In summary, although decreasing the total number $M$ of sampled points can result in shorter runtimes, it increases the conservativeness of reachability analysis which may not allow Algorithm \ref{algo1} to compute a feasible controller even if it exists.

\subsection{Hardware Experiments}\label{sec:hware_exp}

\begin{figure}[t]
    \centering
    \begin{subfigure}[b]{0.41\linewidth}
        \includegraphics[width=\textwidth]{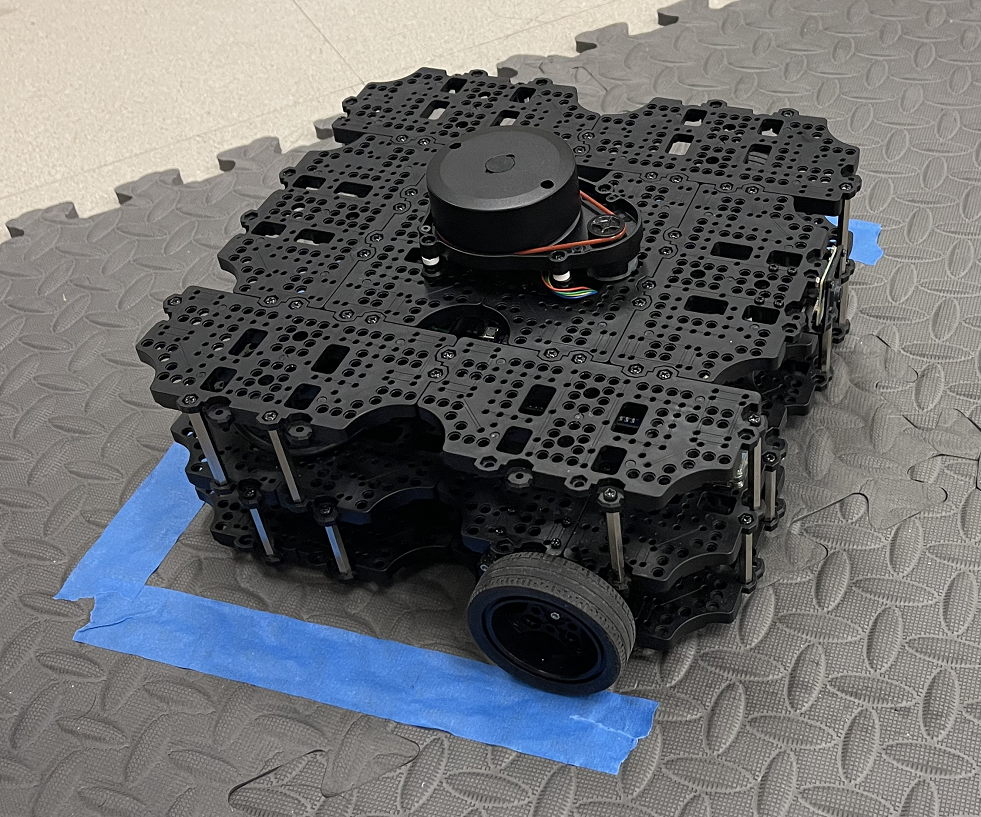}
        \caption{TurtleBot3 Waffle Pi}
        \label{fig:robot}
    \end{subfigure}
    \begin{subfigure}[b]{0.5\linewidth}
        \includegraphics[width=\textwidth]{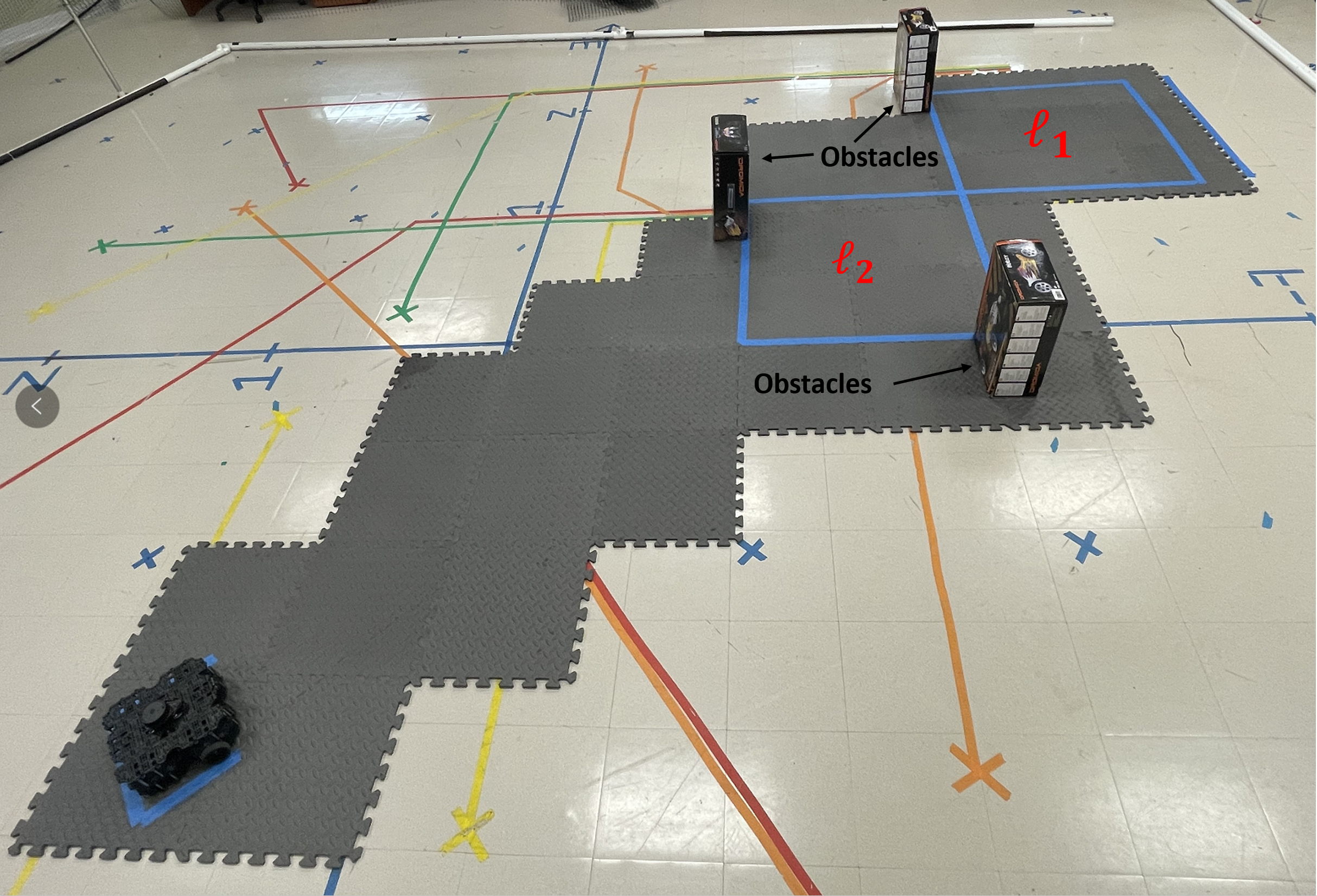}
        \caption{Experiment Environment}
        \label{fig:env}
    \end{subfigure}
    \caption{Turtlebot and environment setup for hardware experiment}\label{fig:expRobEnv}
\end{figure}

Next, we demonstrate the proposed algorithm on hardware experiments. Specifically, we consider a TurtleBot3 Waffle Pi ground robot operating in an indoor environment shown in Fig. \ref{fig:expRobEnv}. %in Appendix \ref{sec:appendix_detail}.
The robot is responsible for accomplishing the following LTL-encoded task 
\begin{align}
    &\phi=(\Diamond\pi^{\ell_1})\wedge (\Diamond\pi^{\ell_2}) \wedge (\neg \pi^{\ell_1} \ccalU \pi^{\ell_2}) \wedge \nonumber\\&(\neg (\pi^{\ell_4} \vee \pi^{\ell_5} \vee \pi^{\ell_6})\ccalU \pi^{\ell_1}) \wedge (\neg (\pi^{\ell_4} \vee \pi^{\ell_5} \vee \pi^{\ell_6})\ccalU \pi^{\ell_2}), \nonumber
\end{align}
requiring it to visit region $\ell_2$ and $\ell_1$ in this  order while avoiding obstacles $\ell_4$, $\ell_5$ and $\ell_6$. This formula corresponds to the DFA shown in Fig. \ref{fig:hard_dfa}. In what follows, we consider the same NN controllers used for the numerical experiments in Section \ref{sec:grobot}. %\textcolor{red}{[Section X?]}.
We apply $\epsilon$-RandUP using $\epsilon=0.03$ and $M=600,000$. As discussed in Section \ref{sec:grobot}, this ensures that the constructed reachable sets over-approximate the true ones with a probability greater than  $0.9996$. The set $\ccalX_0$ is designed so that it collects all possible robot states with $x^1_0\in[2.95, 3.05]$, $x^2_0\in[2.95, 3.05]$, and $\theta_0\in[3.92, 3.94]$. To account for the robot size, in our reachability analysis, we shrink and inflate the regions of interest and obstacles, respectively, by the radius of the robot.

Our method, first, investigates DFA transition from $q_D^0$ to $q_D^1$. This transition requires the ground robot to reach region $\ell_2$ while avoiding regions $\ell_1$, $\ell_4$, $\ell_5$, and $\ell_6$. \textcolor{black}{Thus, we have $\Xi_{q_D^0\rightarrow q_D^1} = \{\xi_2\}$.} The reachable sets $\bar{\ccalR}_t(\ccalX_0,\xi_2, V)$, corresponding to this DFA transition, 
are fully outside the obstacle regions over the next 22 time steps (satisfying the self-loop transition of $q_D^0$) and go fully inside $\ell_2$ at the $23$ time step; see Fig. \ref{fig:g3_1} in Appendix \ref{sec:appendix_detail}. As a result, this DFA transition is verified to be safe with probability at least equal to $0.9996^{23}$.
% while at $t=23$ the corresponding reachable set is fully inside $\ell_2$. 
% So the probability to satisfy this DFA transition is $0.9839$ based on the $\epsilon$ and $N$. 
%
Next, we investigate the transition from $q_D^1$ to $q_D^F$ which is activated if the robot avoids $\ell_1$, $\ell_4$, $\ell_5$, and $\ell_6$ until it reaches $\ell_1$ \textcolor{black}{where  $\Xi_{q_D^1\rightarrow q_D^F} = \{\xi_1\}$.} The set of initial states for this transition will be $\ccalX_{q_D^1}^0=\bar{\ccalR}_{23}(\ccalX_0,\xi_2,V)$ where we sample $M=600,000$ points again within $\ccalX_{q_D^1}^0$. % \textcolor{red}{[Explain where this power of two comes from unless it is a mistake!]} 
The reachable sets over the next $19$ time steps are outside the obstacles and reach $\ell_1$ at time step $20$; see Fig. \ref{fig:g3_2} in Appendix \ref{sec:appendix_detail}. 
%
%at the time steps $t=1\sim19$ the reachable sets are outside the obstacles while at $t=20$ the set is inside $\ell_1$. 
%
%\textcolor{red}{[are these time instants correct? You reach $q_D^1$ at $t=23$. Do we reset the time when we reach a new DFA state? If so, is this discussed anywhere in the paper? Maybe you should avoid mentioning these time steps here to avoid any confusion. Instead, just mention the horizon (e.g., say that the reachable sets over the next H time steps are outside the obstacles). The same also applies to the simulations in the previous section. ]} 
%
Hence, the DFA transition from $q_D^1$ to $q_D^F$ is also verified with probability that is at least equal to $0.9996^{20}$. 
Using these \textit{simulated} results, the robot will satisfy the assigned LTL task if it applies controller $\xi_2$ for $23$ time units and then $\xi_1$ for $20$ time units %. Similarly, we can conclude that 
with probability
$\mathbb{P}_{\tau(\bbx_0)\sim D}(\tau(\bbx_0)\models\phi|\boldsymbol\xi,\ccalX_0)\textcolor{black}{\geq}(0.9996)^{23+20}=0.9829$.

% \begin{figure}[!t]
%     \centering
%     \includegraphics[width=0.5\linewidth]{img/ground_c3_oil_new.png}
%     % \caption{$q_D^0 \rightarrow q_D^1$ ($\theta_0 \in [3.92, 3.94]$) }
%     \caption{Verification of DFA transitions for Section \ref{sec:hware_exp}. All reachable sets that are outside/inside the regions to reach are shown with yellow/green color. Reachable sets that are hitting the regions to avoid are shown with red color.
%     %\textcolor{red}{[what about the red sets?]} 
%     The red line captures robot trajectory. This figure shows that with the oily floor which is not captured in the dynamics, the actual trajectory of the robot deviated from the reachable set and hit the obstacle.
%     \label{fig:g3_4}
%     }
% \end{figure}

To validate this result in real-world settings, we select at test time $11$ initial robot states and apply the designed control strategy $\boldsymbol{\xi}$ on the actual robot platform. 
In our experiments, the robot localizes itself using existing SLAM algorithms \cite{thale2020ros}.
The robot trajectories show that the LTL task is indeed satisfied for all selected initial states $\bbx_0\in\ccalX_0$; see Figures \ref{fig:g3_1} and \ref{fig:g3_2} in Appendix \ref{sec:appendix_detail}. Observe also in these figures (especially in Fig. \ref{fig:g3_2}) that some robot trajectories are slightly outside the reachable sets. We conjecture that this happens due to potential sim-to-real gap arising from imperfect robot localization and model inaccuracies; recall that the controller $\boldsymbol{\xi}$ is designed using a simulator of the robot dynamics as defined in \eqref{eq:grob}.

Next, we further investigate how the sim-to-real gap may affect robot performance with respect to satisfaction of the LTL task.
Specifically, to simulate a hard-to-model source of uncertainty in the environment, we put oil on one of the robot wheels. We select the initial state of the robot to be $\bbx_0=[3; 3; 3.925]$.
% \textcolor{red}{[mention also the initial robot state if possible. If not available delete this sentence]}. 
%
% As discussed above, given the considered robot model in \eqref{eq:grob}, the probability that the constructed reachable sets over-approximate the true ones is $0.9681$ \textcolor{red}{[It is still unclear to me how this probability is computed. Also, why is it important to mention it again here?]}. 
%
% Using reachability analysis, the noise threshold that $q_D^0 \rightarrow q_D^1$ is $\forall i \in \{1,\dots,3\}, \boldsymbol\nu_i \in [-0.0112, 0.0112]$. 
%
% \textcolor{red}{[I dont understand this sentence. Also, I am not sure what is happening here. Was a DFA transition verified to be safe in sims but not in reality? If so, this needs to be mentioned.]} 
%
Notice that the upper bound for the noise $V=0.002$, used in our simulations, may not be enough to capture such a potentially unknown source of uncertainty. This may violate the probabilistic guarantees of the controller $\boldsymbol\xi$ designed above. The latter was validated in our experiments as the robot trajectory was outside the reachable sets, computed earlier using simulated data, leading to collision with the obstacle $\ell_4$; see Fig. \ref{fig:g3_4} in Appendix \ref{sec:appendix_detail}.
%In Fig. \ref{fig:g3_4}, as the reachability analysis under the condition of $V=0.002$ verifies the transition $q_D^0 \rightarrow q_D^1$ to be true, the robot deployed in the world with unbounded disturbance (oily floor) hit the obstacle $\ell_4$ (see the red line). 
% In Fig. \ref{fig:g3_4}, the robot fails to satisfy the LTL specification in real-world given the unbounded disturbance (oily floor) where users usually do not have enough knowledge to model the dynamics of it.
%
% \footnote{\textcolor{magenta}{again here you need mre details. What is the upper bound for noise? In your reachability analysis, what is the maximum upper bound you could pick under which the system was still unsafe?}} 
%
% \footnote{All hardware experiment videos can be found  \href{https://drive.google.com/drive/folders/16XAWUMxktYriX1xYqaa-5JJzqt145z8G?usp=sharing}{HERE}} % \textcolor{red}{[I dont see what exactly is shown in that figure. I don't see any `failure' there.]}

\section{Conclusion}
%\textcolor{red}{[Update the conclusions.]}
% In this paper, we proposed a new verification method  We demonstrated its efficiency with extensive numerical and hardware experiments.

In this paper, we propose a new verified neuro-symbolic control approach for autonomous systems with stochastic and unknown dynamics with temporal logic tasks. The proposed method enables systems to compose base skills, modeled by NN controllers, to verifiably satisfy unseen and complex tasks captured by LTL formulas. We evaluated the proposed method on robot navigation tasks though extensive experiments. %both numerical and hardware experiments to demonstrate the effectiveness of the proposed method.
%the first verified temporal compositions of NN controllers for \textbf{unknown} and \textbf{stochastic} systems with temporal logic tasks. We conduct both numerical and hardware experiments to demonstrate the effectiveness of the proposed method.

\bibliographystyle{IEEEtran}
\bibliography{YK_bib.bib}

\clearpage

\begin{appendices}

\section{Additional Details for Experiments of Section \ref{sec:sim}}
\label{sec:appendix_detail}
% \textcolor{red}{The appendix cannot be just figures and tables. You need some text describing what is happening here. Think of the appendix as a section that is independent of the paper and vice versa in the sense that someone should be able to understand the paper without reading the appendix and vice versa. }

In this Appendix, we provide additional details regarding the experimental validation discussed in Section \ref{sec:sim}.

% \subsection{DFA Pre-Processing}

% In Fig \ref{fig:pre_processing_dfa}, we demonstrate the resulting DFA graph for LTL task $\phi=(\Diamond \pi^{\ell_1}) \wedge (\Diamond \pi^{\ell_2}) \wedge (\neg \pi^{\ell_1} \ccalU \pi^{\ell_2}) \wedge (\neg (\pi^{\ell_3} \vee \pi^{\ell_4} \vee \pi^{\ell_5}) \ccalU \pi^{\ell_1})$ after the pre-processing process described in Section \ref{sec:vltl}.

% \begin{figure}[!ht]
%     \centering
%     \includegraphics[width=\linewidth]{img/wustl_dfa_extend.png}
%     \caption{Illustration of the DFA pre-processing step as the number of incoming edges (excluding self-loops) for $q_D^F$ is greater than 1}
%     \label{fig:pre_processing_dfa}
% \end{figure}
\vspace{-0.2cm}
\subsection{NN Controller Performance}
In Table \ref{tab:perf}, we report the performance of the NN controllers used throughout Section \ref{sec:sim} over a test set collecting initial system states. 
Specifically, we apply the NN controller over these initial test-time system states and we check if the system state can reach the desired region within the same horizon length as we used for the MPC solver. If so, we treat these test-time scenarios as successful. The accuracy reported in Table \ref{tab:perf} refers to the percentage of successful test-time scenarios. %Hereafter, we used these NN controllers unless otherwise stated. 
%
% \textcolor{red}{[Put here all related details about them,]}

\begin{table}[!ht]
\centering
\begin{tabular}{|c|c|c|c|c|}
\hline
                              & Controller & Training Set      & Test Set      & Accuracy \\ \hline
\multirow{3}{*}{Ground Robot} & $\xi_1$                & \multirow{3}{*}{4800}  & \multirow{3}{*}{1200} & $99.8\%$ \\ \cline{2-2} \cline{5-5} 
                              & $\xi_2$                &                        &                       & $83.3\%$ \\ \cline{2-2} \cline{5-5} 
                              & $\xi_3$                &                        &                       & $79.2\%$ \\ \hline
\end{tabular}
\caption{NN Controller Performance}
\label{tab:perf}
\end{table}

\vspace{-0.2cm}
\subsection{Reachable Sets for Numerical Experiment}
Next, in Fig. \ref{fig:g1los}, we graphically illustrate the verification of DFA transitions corresponding to the task of Section \ref{sec:grobot} when $M=50,000$. Recall that to maintain the same probabilistic guarantees as in $M=600,000$, the padding size $\epsilon$ increased to $0.1$. This resulted in larger reachable sets compared to the ones in Fig. \ref{fig:g1suc}.

%Next, in Figure \ref{fig:g1los}, we illustrate the reachable sets used for the  verification of the DFA transition from $q_D^0$ to $q_D^1$ in the case study considered in Section \ref{sec:nu}.

\begin{figure}[h]
    \centering
    \includegraphics[width=0.8\linewidth]{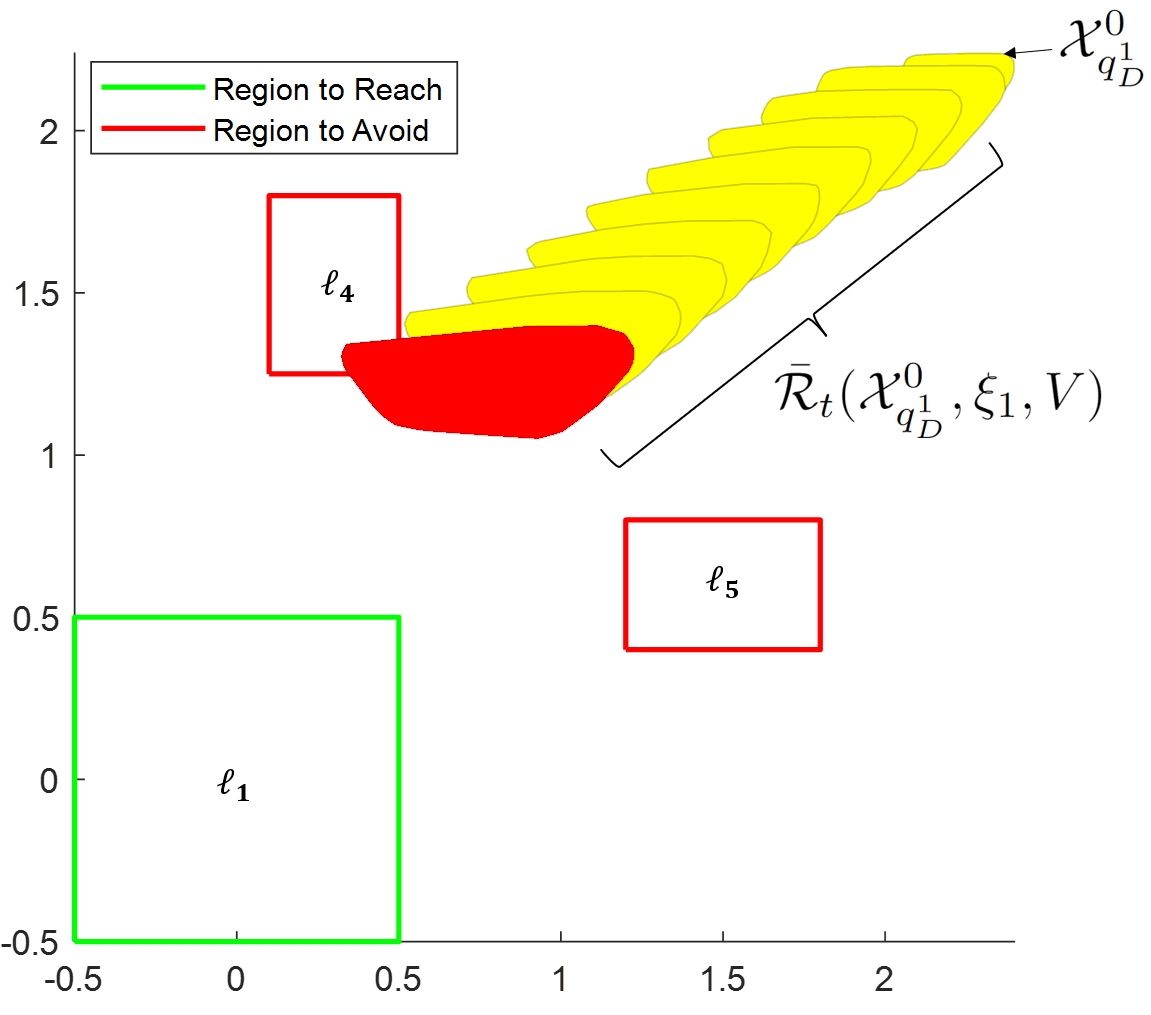}
    \caption{$q_D^1\rightarrow q_D^F$ ($\epsilon=0.1$)}
    \label{fig:g1los}
    \caption{Verification of a DFA transition for Section \ref{sec:grobot} - $M=50,000$. All reachable sets that are outside/inside the regions-to-reach are shown with yellow/green color. Reachable sets that are hitting the regions-to-avoid are shown with red color.
    %\textcolor{red}{[what about the red sets?]} 
    The red line captures robot trajectory. This figure shows that $q_D^1 \rightarrow q_D^F$ is verified as unsafe when $\epsilon$ increases. % Figure \ref{fig:g3_4} shows that with the oily floor which is not captured in the dynamics, the actual trajectory of the robot deviated from the reachable set and hit the obstacle.\textcolor{red}{[Make sure the figs appear in the  order they are discussed in the text. The robot figure was before this fig in the previous draft. Also, make the captions helpful. E.g., add a note so that it is clear to what section these results refer to. I would place Fig 5b somewhere else by itself (maybe in the appendix). It causes some confusion to have side by side `irrelevant' figs with reachable sets for different case studies. If you have the figure for the other DFA transition, for the numerical experiments, put it here as well in place of 5b.  ] }
    }
\end{figure}

\subsection{Reachable Sets for Hardware Experiments}
Next, in Figure \ref{fig:hardwareexp},  we present the reachable sets computed for the hardware experiment discussed in Section \ref{sec:hware_exp}. Observe that both DFA transitions are verified to be safe using the robot simulator. Also, observe in this Figure that the trajectories generated by the actual robot platform remain inside the reachable sets (with small deviations as discussed in Section \ref{sec:hware_exp}). Figure \ref{fig:g3_4} shows a robot trajectory when the robot is subject to exogenous disturbances (oily wheels) that are not modeled by the system simulator violating the probabilistic safety guarantees computed using simulated data. 
%Next, we graphically illustrate the verification of DFA transitions corresponding to the task discussed in Section \ref{sec:grobot} when $M=50,000$. Recall that to maintain the same probabilistic guarantees as in $M=600,000$, the padding size $\epsilon$ increased to $0.1$. This has resulted in larged reachable sets compared to the ones in Fig. \ref{fig:g1suc}.
\begin{figure}[h]
    \centering
    \begin{subfigure}[b]{0.48\linewidth}
        \includegraphics[width=\textwidth]{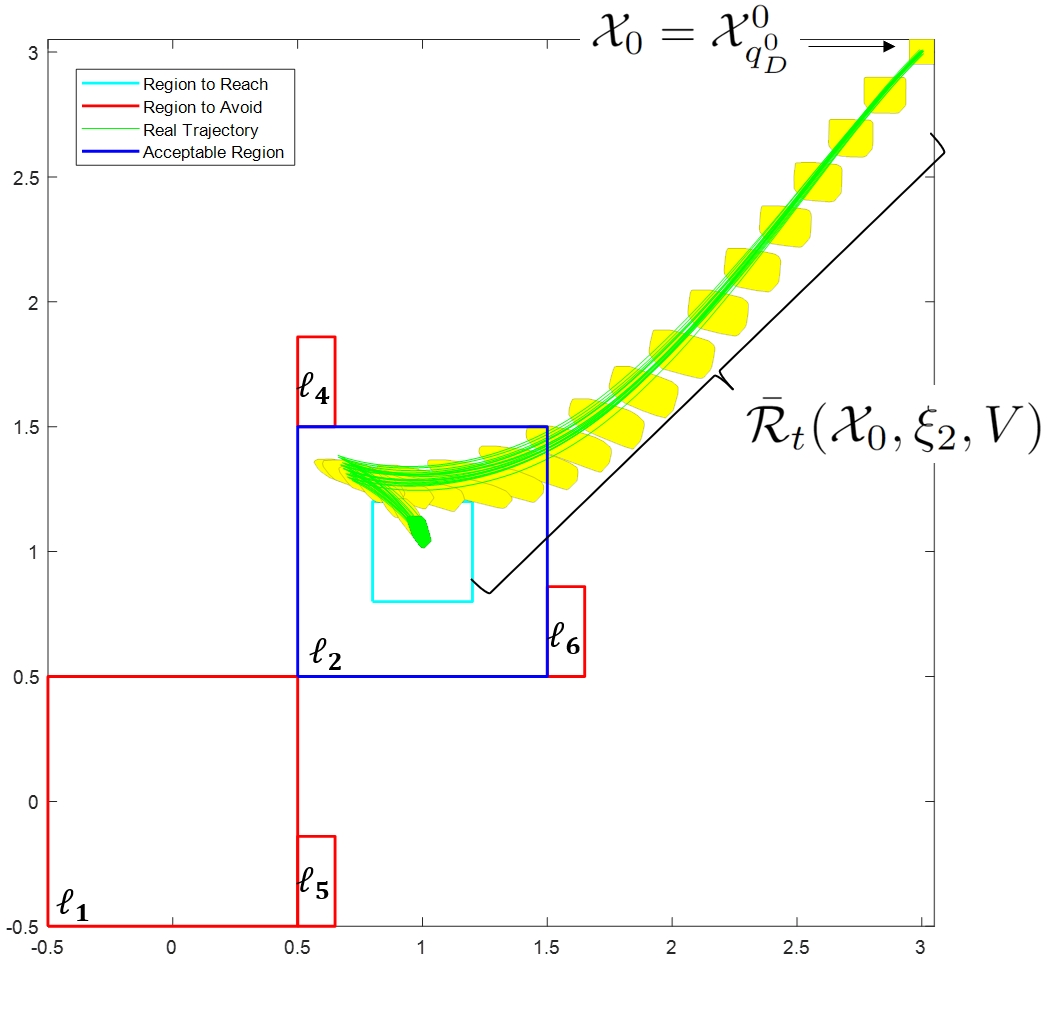}
        \caption{$q_D^0 \rightarrow q_D^1$ }
        \label{fig:g3_1}
    \end{subfigure}
    \begin{subfigure}[b]{0.44\linewidth}
        \includegraphics[width=\textwidth]{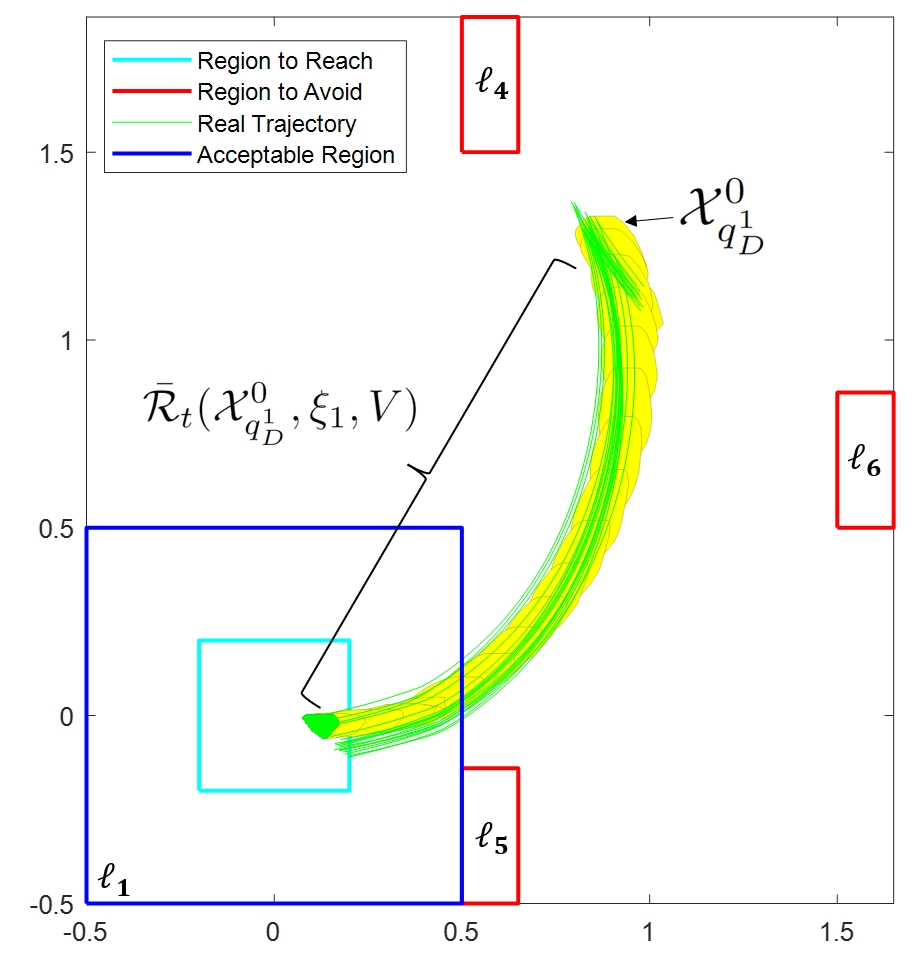}
        \caption{$q_D^1 \rightarrow q_D^F$}
        \label{fig:g3_2}
    \end{subfigure}
    \caption{Verification of DFA transitions for the hardware experiment with $\theta_0 \in [3.92, 3.94]$; The regions of interest that need to be avoided/reached to enable a DFA transition are shown with red/blue color. All other regions are not shown. All reachable sets that are outside/inside the regions to reach are shown with yellow/green color. The green lines capture robot trajectories corresponding to various initial states.}
    \label{fig:hardwareexp}
\end{figure}

% \begin{figure}[!h]
%     \centering
%     \includegraphics[width=0.8\linewidth]{img/ground_c3_success_1_new.png}
%     \caption{$q_D^0 \rightarrow q_D^1$ }
%     \label{fig:g3_1}
%     \caption{Verification of DFA transition for hardware experiment with $\theta_0 \in [3.92, 3.94]$; The regions of interest that need to be avoided/reached to enable this DFA transition are shown with red/blue color. All other regions are not shown. All reachable sets that are outside/inside the regions to reach are shown with yellow/green color. The green lines capture robot trajectories corresponding to various initial states.}
%     %\label{fig:hardwareexp}
% \end{figure}

% \begin{figure}[!h]
%     \centering
%     \includegraphics[width=0.8\linewidth]{img/ground_c3_success_2_new.png}
%     \caption{$q_D^1 \rightarrow q_D^F$}
%     \label{fig:g3_2}
%     \caption{Verification of DFA transition for hardware experiment with $\theta_0 \in [3.92, 3.94]$; The regions of interest that need to be avoided/reached to enable this DFA transition are shown with red/blue color. All other regions are not shown. All reachable sets that are outside/inside the regions to reach are shown with yellow/green color. The green lines capture robot trajectories corresponding to various initial states.}
%     %\label{fig:hardwareexp}
% \end{figure}

\begin{figure}[h]
    \centering
    \includegraphics[width=0.8\linewidth]{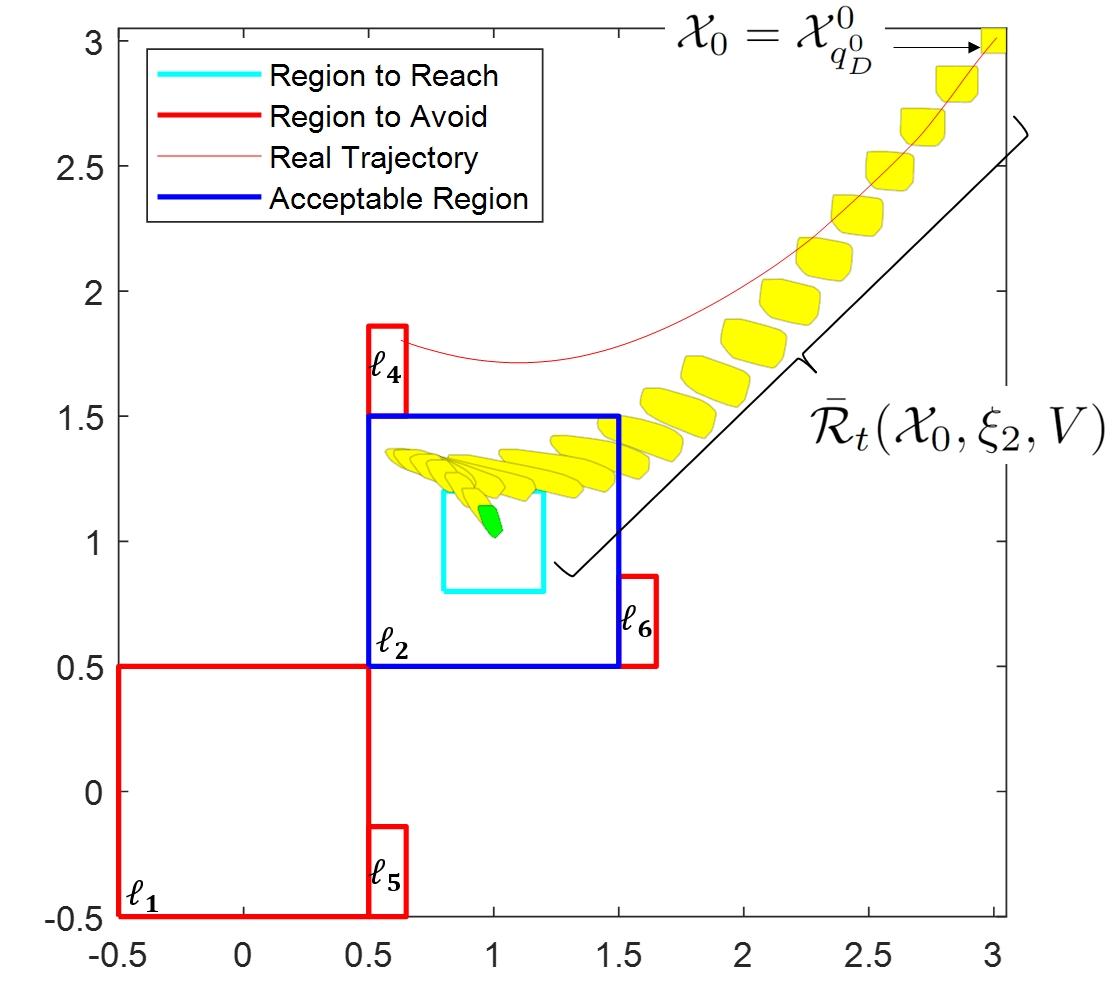}
    \caption{$q_D^0 \rightarrow q_D^1$ ($\theta_0 \in [3.92, 3.94]$) }
    \label{fig:g3_4}
    \caption{Verification of a DFA transition for the hardware experiment when the robot is subject to unmodeled disturbances. All reachable sets that are outside/inside the region-to-reach are shown with yellow/green color. Reachable sets that are hitting the regions-to-avoid are shown with red color.
    %\textcolor{red}{[what about the red sets?]} 
    The red line captures robot trajectory when the robot is subject to exogenous disturbances that are not captured by the system model/simulator. Observe that the actual trajectory of the robot is outside the reachable sets and eventually leads to collision with obstacle $\ell_4$.
    }
\end{figure}

\clearpage
\section{Additional Numerical Experiment} \label{sec:appendix_extra}

In this appendix, we provide an additional numerical experiment.
We consider the following LTL formula 
\begin{align}
    &\phi=(\Diamond\pi^{\ell_1})\wedge (\Diamond\pi^{\ell_2}) \wedge (\Diamond\pi^{\ell_3})\wedge (\neg \pi^{\ell_2} \ccalU \pi^{\ell_3}) \wedge (\neg \pi^{\ell_1} \ccalU \pi^{\ell_2}) \wedge\nonumber\\ &(\neg \pi^{\ell_4} \ccalU \pi^{\ell_1}) \wedge (\neg \pi^{\ell_5} \ccalU \pi^{\ell_1}),\nonumber
\end{align}
requiring the ground robot to eventually visit the regions $\ell_1$, $\ell_2$ and $\ell_3$, in the order of $\ell_3 \rightarrow \ell_2 \rightarrow \ell_1$, while avoiding the obstacles $\ell_4$ and $\ell_5$.
This formula corresponds to the DFA shown in Fig. \ref{fig:g2first}; notice that the transition from $q_D^0$ to $q_D^{F}$, $q_D^1$ to $q_D^{F}$, and $q_D^0$ to $q_D^2$ were pruned.
The initial set $\ccalX_0$ of system states is defined as $\{ \bbx_0 \in \ccalX | x_0^1\in [3.7, 3.8]; ~x_0^2 \in [3.5, 3.6];~ \theta_0 \in [4.02, 4.04] \}$.
To compute the reachable sets, we apply $\epsilon$-RandUP using $M=600,000$ sampled points for the initial set of states with padding size of $\epsilon=0.03$.
\begin{figure}[h]
\centering
\includegraphics[width=\linewidth]{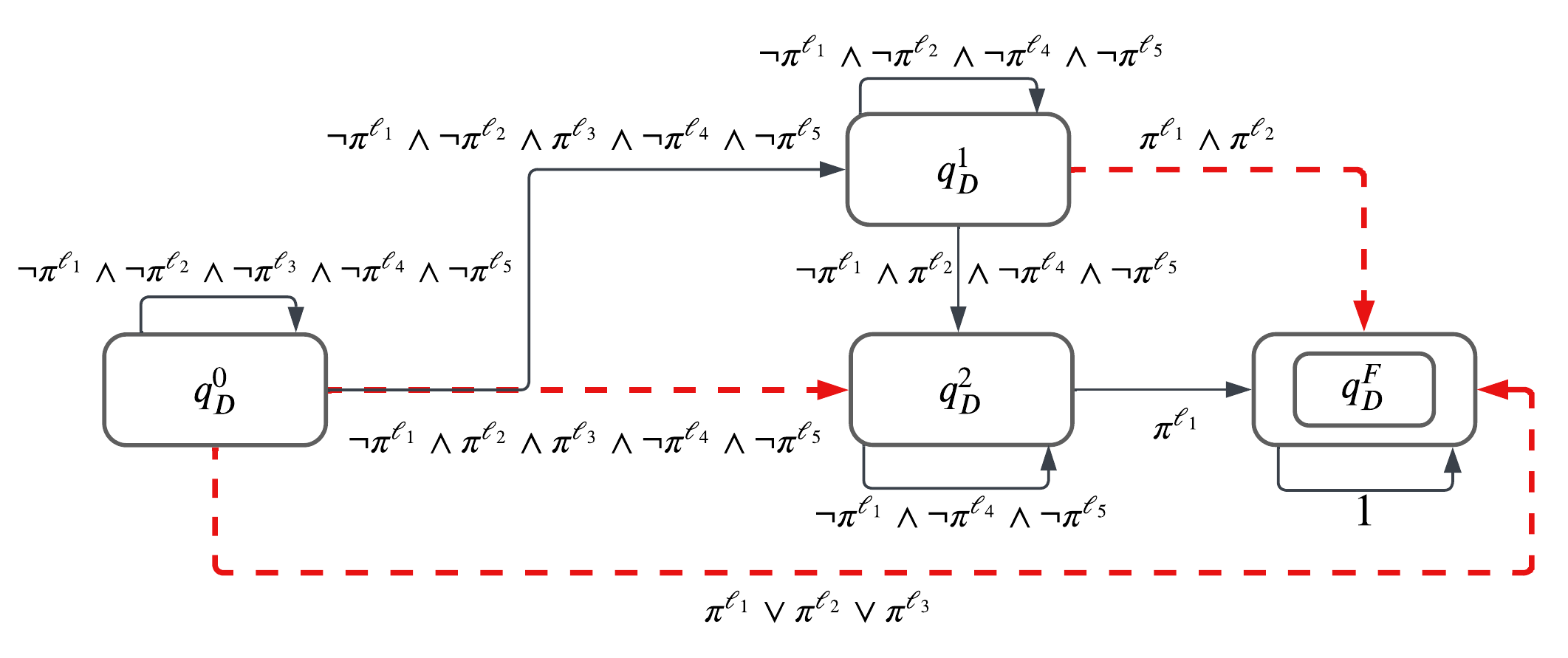}
\caption{DFA for the case study in Appendix \ref{sec:appendix_extra}. The red dashed transitions correspond to infeasible transitions.}
\label{fig:g2first}%\vspace{-5mm}
\end{figure}

Our method, first, investigates the DFA transition from $q_D^0$ to $q_D^1$. This transition requires the ground robot to stay within the obstacle-free space (i.e., avoid the obstacles $\ell_1$, $\ell_2$, $\ell_4$ and $\ell_5$) and eventually reach $\ell_3$. \textcolor{black}{Therefore, we have that $\Xi_{q_D^0\rightarrow q_D^1} = \{\xi_3\}$.}
The corresponding reachable sets for this DFA transition are shown in Figure \ref{fig:g2suc}. Notice that the reachable sets $\bar{\ccalR}_t(\ccalX_0,\xi_3, V)$ are fully outside the obstacle regions for $\bar{\ccalR}_1(\ccalX_0,\xi_3, V) \sim \bar{\ccalR}_{10}(\ccalX_0,\xi_3, V)$ while $\bar{\ccalR}_{11}(\ccalX_0,\xi_3, V)$ is fully inside $\ell_3$. Thus, this DFA transition is verified to be safe with the probability at least equal to $(0.9996)^{11}$. Next, the DFA transition from $q_D^1$ to $q_D^2$ is considered requiring the robot to reach $\ell_2$ while avoiding the obstacle regions $\ell_1$, $\ell_4$ and $\ell_5$. \textcolor{black}{Thus, we have that  $\Xi_{q_D^1\rightarrow q_D^2} = \{\xi_2\}$.}
The set of initial states for this transition is $\ccalX_{q_D^1}^0=\bar{\ccalR}_{11}(\ccalX_0,\xi_3, V)$. After computing the padded reachable convex hulls $\bar{\ccalR}_t(\ccalX_{q_D^1}^0,\xi_2, V)$  shown in Figure \ref{fig:g2suc3}, this transition is verified to be safe with the probability of $(0.9996)^{15}$. Similarly, we can show  that the transition from $q_D^2$ to $q_D^F$  is also safe with probability at least equal to $(0.9996)^{11}$; see Figure \ref{fig:g2sucf}.
%(\textcolor{black}{$\Xi_{q_D^2\rightarrow q_D^F} = \{\xi_1\}$.})
We conclude that there exists a control strategy $\boldsymbol\xi=\xi_3,\xi_2,\xi_1$ with $H_0=11, H_1=15,$ and $H_3=11$, such that  $\mathbb{P}_{\tau(\bbx_0)\sim D}(\tau(\bbx_0)\models\phi|\boldsymbol\xi,\ccalX_0)=(0.9996)^{11+15+11}\geq0.9853$.

% \textcolor{red}{Put this in a different appendix. In the main text, make sure you make a reference to the correct appendix.}

\begin{figure}[!ht]
    \centering
    \begin{subfigure}[b]{0.51\linewidth}
        \includegraphics[width=\textwidth]{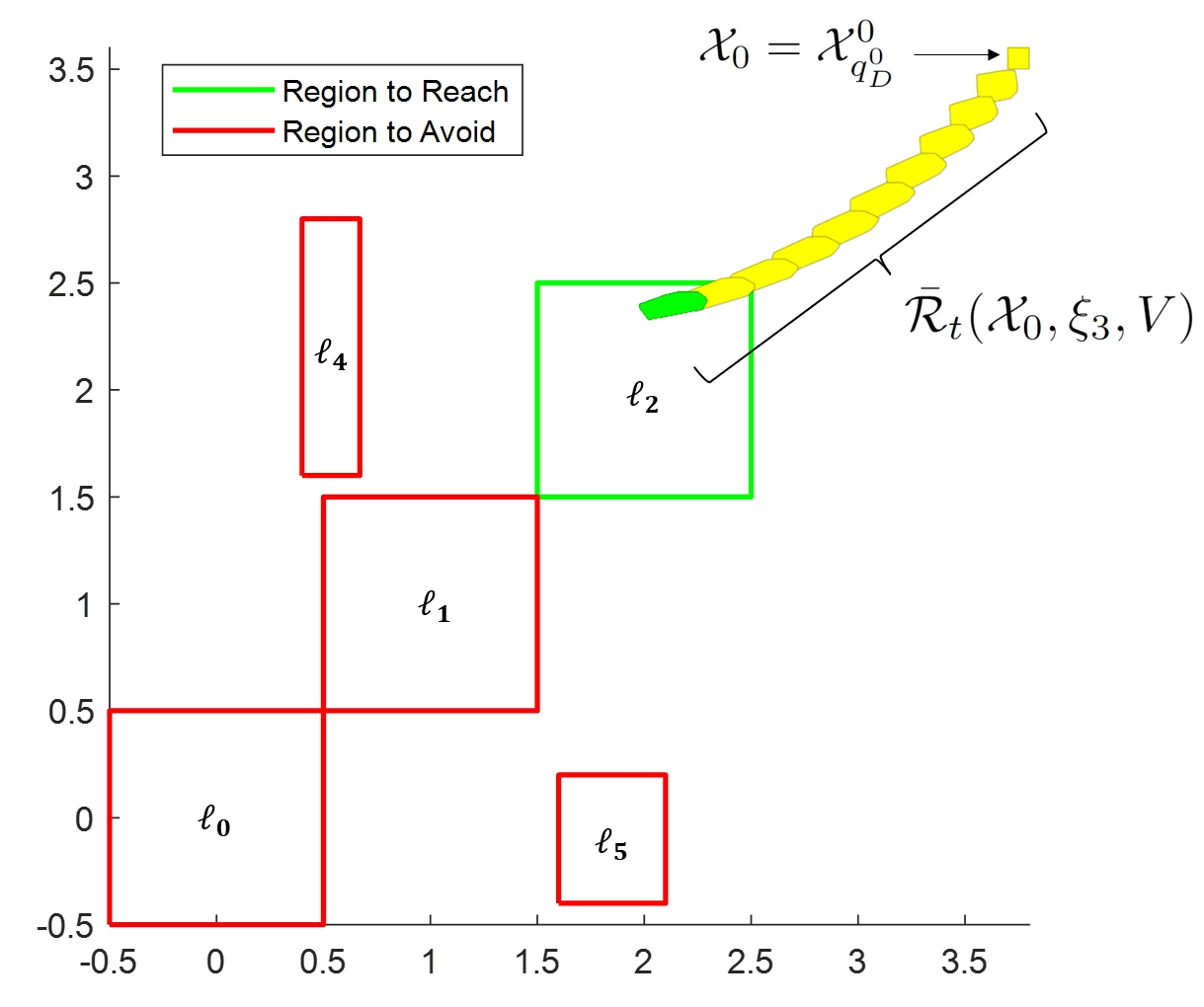}
        \caption{$q_D^0\rightarrow q_D^1$ ($\epsilon=0.03$) }
        \label{fig:g2suc}
    \end{subfigure}
    \begin{subfigure}[b]{0.47\linewidth}
        \includegraphics[width=\textwidth]{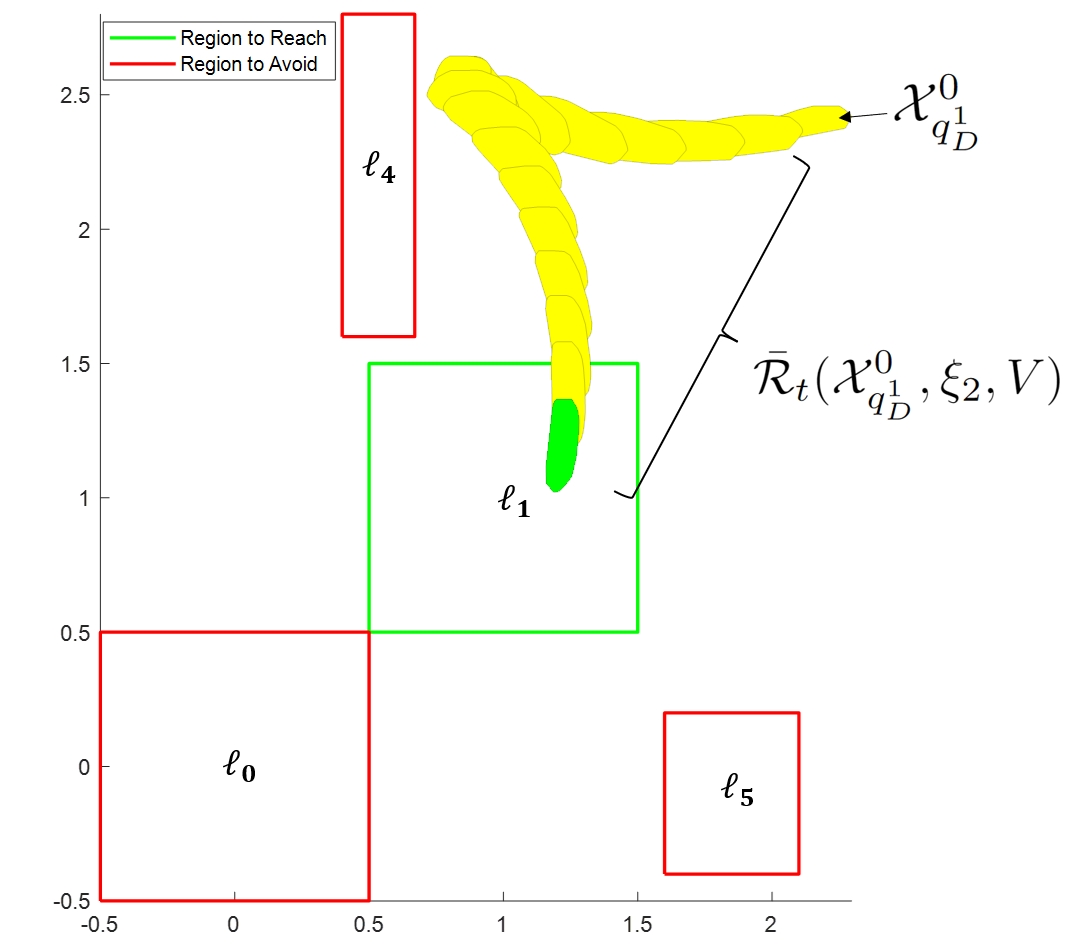}
        \caption{$q_D^1\rightarrow q_D^2$ ($\epsilon=0.03$) }
        \label{fig:g2suc3}
    \end{subfigure}
    \caption{Verification of DFA transition $q_D^0\rightarrow q_D^1$ and $q_D^1\rightarrow q_D^2$ for the numerical experiment discussed in Appendix \ref{sec:appendix_extra}.}
    \label{fig:g2_2}
\end{figure}

\begin{figure}[!ht]
    \centering
    \includegraphics[width=0.8\linewidth]{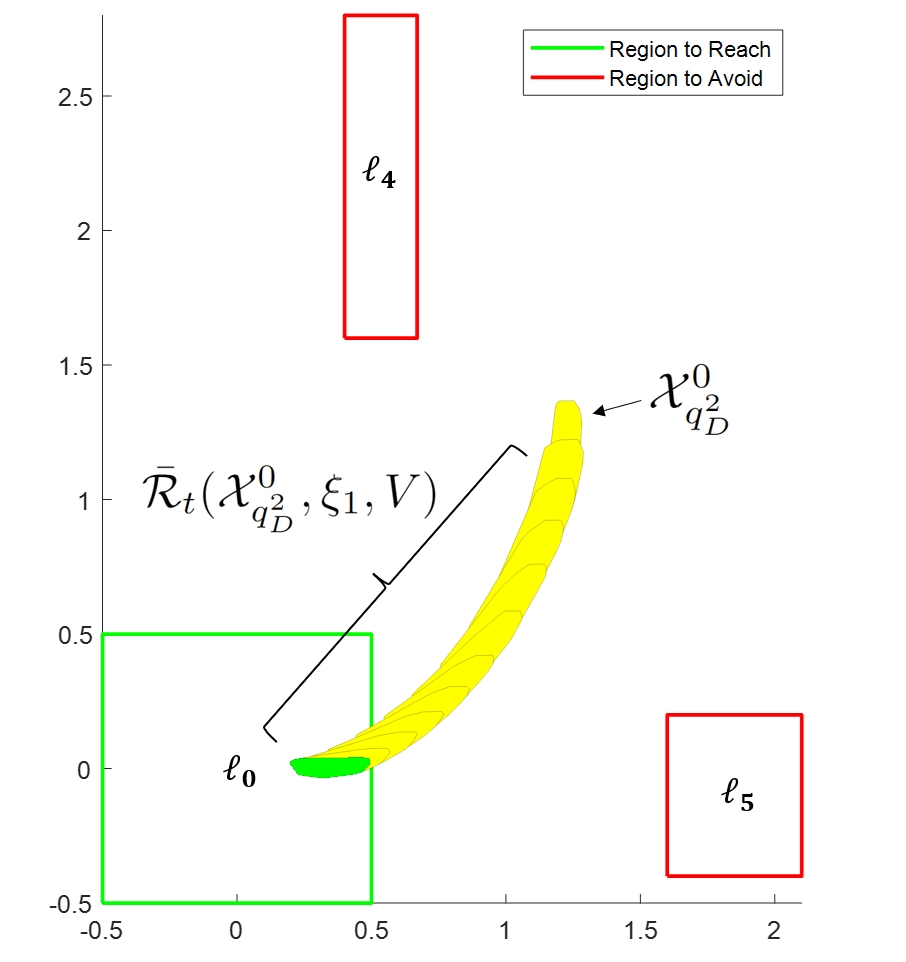}
    % \caption{$q_D^2\rightarrow q_D^F$}
    \caption{Verification of DFA transition $q_D^2\rightarrow q_D^F$ for the numerical experiment discussed in Appendix \ref{sec:appendix_extra}.}
    \label{fig:g2sucf}
\end{figure}

% \begin{figure}[!ht]
%     \centering
%     \begin{subfigure}[b]{0.44\linewidth}
%         \includegraphics[width=\textwidth]{img/ground_c2_success_f_new.png}
%         \caption{$q_D^2\rightarrow q_D^F$}
%         \label{fig:g2sucf}
%     \end{subfigure}
%     \begin{subfigure}[b]{0.44\linewidth}
%         \includegraphics[width=\textwidth]{img/ground_c2_fail_3.png}
%         \caption{$q_D^1\rightarrow q_D^2$ ($\epsilon=0.1$)}
%         \label{fig:g2los}
%     \end{subfigure}
%     \caption{Verification of DFA transition $q_D^2\rightarrow q_D^F$ and $q_D^2\rightarrow q_D^F$ for the case studies II.}
%     \label{fig:g2_2}
% \end{figure}
\end{appendices}

\end{document}